\documentclass[12pt]{article}

\usepackage{fullpage}
\usepackage[utf8]{inputenc} 
\usepackage[T1]{fontenc}    
\usepackage{hyperref}       
\usepackage{url}            
\usepackage{booktabs}       
\usepackage{amsfonts}       
\usepackage{nicefrac}       
\usepackage{microtype}      
\usepackage{microtype}
\usepackage{graphicx}
\usepackage{float}
\usepackage{soul}
\usepackage{caption}
\usepackage{bbding}
\usepackage{float}
\usepackage{sidecap}
\usepackage{enumitem}
\usepackage{amsmath,amssymb,amsthm}
\usepackage{color}
\usepackage{algorithm,algorithmic}
\usepackage{amsfonts}
\usepackage{subfigure}

\usepackage{hyperref}
\usepackage{amsmath}
\usepackage{booktabs}
\usepackage{hyperref}
\usepackage[margin=1in]{geometry}
\usepackage{setspace}
\usepackage{bm}
\usepackage{algorithm}
\usepackage{algorithmic}
\numberwithin{equation}{section}
\usepackage{chngcntr}
\usepackage{apptools}
\AtAppendix{\counterwithin{theorem}{section}}

\newcommand{\BEAS}{\begin{eqnarray*}}
\newcommand{\EEAS}{\end{eqnarray*}}
\newcommand{\BEQ}{\begin{equation}}
\newcommand{\EEQ}{\end{equation}}
\newcommand{\BIT}{\begin{itemize}}
\newcommand{\EIT}{\end{itemize}}

\newcommand{\eg}{{\it e.g.}}
\newcommand{\ie}{{\it i.e.}}

\newcommand{\iid}{{\it i.i.d.}}

\newcommand{\reals}{{\mbox{\bf R}}}
\newcommand{\symm}{{\mbox{\bf S}}}

\newcommand{\argmin}{\mathop{\rm argmin}}

\def\<#1,#2>{\langle #1,#2\rangle}

\newcommand{\prox}{{\mathop{\textbf{prox}}}}

\newtheorem{theorem}{Theorem}
\newtheorem{remark}{Remark}
\newtheorem{assumption}{Assumption}[section]
\newtheorem{lemma}[theorem]{Lemma}

\theoremstyle{definition}

{\begin{quote}}{\end{quote}}

\makeatletter
\long\def\@makecaption#1#2{
   \vskip 9pt
   \begin{small}
   \setbox\@tempboxa\hbox{{\bf #1:} #2}
   \ifdim \wd\@tempboxa > 5.5in
        \begin{center}
        \begin{minipage}[t]{5.5in}
        \addtolength{\baselineskip}{-0.95pt}
        {\bf #1:} #2 \par
        \addtolength{\baselineskip}{0.95pt}
        \end{minipage}
        \end{center}
   \else
	\hbox to\hsize{\hfil\box\@tempboxa\hfil}
   \fi
   \end{small}\par
}
\makeatother

\newcounter{oursection}

\newcounter{lecture}

\newcommand{\norm}[1]{{\left\vert\kern-0.25ex\left\vert\kern-0.25ex\left\vert #1
    \right\vert\kern-0.25ex\right\vert\kern-0.25ex\right\vert}}

\author{Ziheng Cheng\thanks{Contributed equally to this work.}
        \thanks{School of Mathematical Sciences, Peking University. 
        \textbf{Email:} \texttt{alex-czh@stu.pku.edu.cn}
        }
       \and
       Junzi Zhang\footnotemark[1] 
       \thanks{Citadel Securities (Work done prior to joining Citadel Securities). 
       \textbf{Email:} \texttt{saslascroyale@gmail.com}}
       \and
       Akshay Agrawal
       \thanks{Marimo Inc. 
       \textbf{Email:} \texttt{akshay@marimo.io}}
       \and 
       Stephen Boyd
       \thanks{Department of Electrical Engineering, Stanford University. \textbf{Email:} \texttt{boyd@stanford.edu}}
       }

\title{Joint Graph Learning and Model Fitting in Laplacian Regularized Stratified Models}

\begin{document}

\date{}
\maketitle

\begin{abstract}
\noindent Laplacian regularized stratified models (LRSM) are models that utilize the explicit or implicit network structure of the sub-problems as defined by the categorical features called strata (\eg, age, region, time, forecast horizon, etc.), and draw upon data from neighboring strata to enhance the parameter learning of each sub-problem.  They 
have been widely applied in machine learning and signal processing problems, including but not limited to time series forecasting, representation learning, graph clustering, max-margin classification, and general few-shot learning. Nevertheless, existing works on LRSM have either assumed a known graph or are restricted to specific applications. 
In this paper, we start by showing the importance and sensitivity of graph weights in LRSM, and provably show that the sensitivity can be arbitrarily large when the parameter scales and sample sizes are heavily imbalanced across nodes. We then propose a generic approach to jointly learn the graph while fitting the model parameters by solving a single optimization problem. We interpret the proposed formulation from both a graph connectivity viewpoint and an end-to-end Bayesian perspective, and propose an efficient algorithm to solve the problem.  Convergence guarantees of the proposed optimization algorithm is also provided despite the lack of global strongly smoothness of the Laplacian regularization term typically required in the existing literature, which may be of independent interest. Finally, we illustrate the efficiency of our approach compared to existing methods by various real-world numerical examples. 
\end{abstract}

\section{Introduction}
\paragraph{Stratified models.} In this paper, we consider fitting stratified models for a group of identified categorical features, namely simultaneously fitting multiple (mutually related) models for each category with a shared base model. These models are ubiquitous in practice, including but not limited to personalized click-through-rate prediction, multi-region weather forecasts and multi-horizon asset pricing, where the categories are user types, regions and forecast horizons, respectively. 

More precisely, 
we are given data records of the form $(z,x,y) \in 
\{1, \ldots, K\} \times \mathcal X \times \mathcal Y$,
where $z$ is the identified categorical feature (which we call strata and takes $K$ possible values) over which we stratify, $x \in \mathcal X$ are features,
and $y \in \mathcal Y$ are outputs or outcomes, and
$\mathcal X$ and $\mathcal Y$ can consist of numerical, categorical, or any other data types. Hence in particular, both classification and regression problems are incorporated. In some cases, we do not have $x$  and thus the data records have the form $(z,y)$.


\paragraph{Laplacian regularized stratified models.}
A Laplacian regularized stratified model (LRSM) is the solution to~\cite{tuck2019distributed}
\BEQ\label{eq:lap_strat}
\begin{array}{ll}
\mbox{minimize}_{\Theta} & 
\sum_{k=1}^K(l_k(\theta_k)+r(\theta_k)) + \mathcal L(\theta_1, \ldots, \theta_K),
\end{array}
\EEQ
with optimization variable $\Theta = [\theta_1 \cdots \theta_K] \in \reals^{n \times K}$.
Here $l_k(\theta)=\sum_{i:z_i=k}l(\theta,x_i,y_i)$ is the $k$th local loss function;
$r: \reals^n \to \reals \cup \{\infty\}$ is the regularization of each model
(where infinite values encode constraints on allowable model parameters). 

The last term in \eqref{eq:lap_strat} is 
Laplacian regularization, which encourages
neighboring values of $z$, under some weighted graph, to have similar
parameters.
It is characterized by $W \in \symm^K$, a symmetric weight matrix with 
zero diagonal entries and nonnegative off-diagonal entries.
The Laplacian regularization has the form
\[
\mathcal L(\Theta) = \mathcal L(\theta_1, \ldots, \theta_K)
=\frac{1}{2} \sum_{i=1}^K\sum_{i < j} W_{ij} \|\theta_i - \theta_j\|^2,
\]
where the norm is the Euclidean or $\ell_2$ norm when $\theta_z$ is a
vector, and the Frobenius norm when $\theta_z$ is a matrix.
We think of $W$ as defining a weighted similarity graph, with edges associated
with positive entries of $W$, and with edge weight $W_{ij}$.
The larger $W_{ij}$ is, the more encouragement
$\theta_i$ and $\theta_j$ have to be close to one another.
We can also write the Laplacian regularization as the positive semidefinite quadratic form
\[
\mathcal L(\Theta) = (1/2) \mathbf{Tr}(\Theta L \Theta^T) = (1/2)\mathbf{Tr}(\Theta^T \Theta L),
\]
and $L = \mathcal{G}(W) \in \reals^{K \times K}$ is the (weighted) Laplacian matrix associated with the weighted graph, where $\mathcal{G}:\reals^{K\times K}\rightarrow\reals^{K\times K}$ 
is a linear mapping defined as follows:
\begin{equation*}
\mathcal{G}(W)_{ij}=\left\{
\begin{array}{ll}
-W_{ij} & i\neq j,\\
\sum\nolimits_{k\neq i}W_{ik} & i=j.
\end{array}
\right.
\end{equation*}

In the literature, it is generally assumed that the graph 
edge weights $W_{ij}$ are non-negative and known \textit{a priori}, 
with larger values indicating stronger closeness between neighboring model parameters. 
However, these weights are typically decided by heuristics, which can be inaccurate and 
arbitrary (especially in terms of the scales) even when a relatively accurate closeness 
relationship between different sub-models is known in advance.  In this paper, we aim to fit the stratified models while learning graph topology behind our data at the same time. We refer to the resulting novel method, which draws upon insights from graph theory and probabilistic/statistical modeling, 
as the \textit{Joint Laplacian stratified model}. 

Throughout the paper, we assume that both $l_k$ and $r$ are closed and proper functions. In addition, $l_k$ is assumed to be differentiable, and $r$ is assumed to have closed-form and easy-to-evaluate proximal operators (\eg, $\ell_1$, ridge and elastic regularization).

\paragraph{Contributions.} Our contributions are three-fold. Firstly, we conduct both theoretical and numerical analysis of the sensitivity of graph weights in LRSM under zero/few-shot settings, which demonstrate the inefficiency of traditional methods that set handmade edge weights, especially when facing severe imbalance of sample sizes and parameter scales. Secondly, we propose a new framework to simultaneously fit LRSM and learn the unknown graph by solving a joint optimization problem over $\Theta$ and $W$ when the graph structure is unavailable. We also give some interpretations of our design from the lens of graph theory and probabilistic models. Moreover, since the resulting optimization problem is generally non-convex and non-smooth, we adopt an accelerated proximal gradient method to solve it and provide a convergence guarantee despite the lack of global strongly smoothness of the Laplacian regularization term, which is typically required in the existing literature. Last but not least, we also demonstrate the efficiency of our method over existing ones with various real-world data examples.  

\section{Motivation: Sensitivity of graph weights}\label{sec:motivation} 


In this section, we aim to show that when parameter scale and sample size are heavily imbalanced across nodes, LRSM suffer from the sensitivity of edge weights even if the graph structure is known. Here sensitivity means that a tiny change to edge weights leads to large errors. With this concern, gradient-based end-to-end optimization of the weights is considerably more efficient than traditional methods like grid search and cross-validation.

Consider the case that a majority of strata have massive data samples and thus can be estimated accurately even without Laplacian regularization. However, a few strata can only get access to few or even no data samples, which is a common setting in few-shot or zero-shot learning. In this way, to fit model parameters at these nodes, a well-designed graph Laplacian is essential for transferring data information. But if the scale of parameters does not vary smoothly on the graph, then a slight perturbation on edge weights will lead to drastic changes in estimated parameters. 

Informally, when a node denoted by $k_0$ has few or no samples, then the output of Laplacian regularized stratified model \eqref{eq:lap_strat} will be
\begin{equation*}
    \widehat{\theta}_{k_0} \approx \frac{\sum_{j\neq k_0} W_{jk_0} \widehat{\theta}_j}{\sum_{j\neq k_0} W_{jk_0}}\approx \frac{\sum_{j\neq k_0} W_{jk_0} \theta_j^{*}}{\sum_{j\neq k_0} W_{jk_0}}
\end{equation*}
provided that all other sub-models are well-fitted. Hence for any $i\neq k_0$, the derivative of $\widehat{\theta}_{k_0}$ with respect to $w_{ik_0}$ is approximately
\begin{equation*}
    \frac{d\widehat{\theta}_{k_0}}{dW_{ik_0}}\approx \frac{\sum_{j\neq k_0, j} W_{jk_0}(\theta_i^*-\theta_j^*)}{(\sum_{j\neq k_0} W_{jk_0})^2}.
\end{equation*}
If $||\theta_{i}^*||$ is very large then $||\frac{d\widehat{\theta}_{k_0}}{dW_{ik_0}}||$ would be large as well and thus $\widehat{\theta}_{k_0}$ is sensitive to the edge weights. We state this point of view more rigorously under zero-shot setting. The proof is differed in Appendix \ref{appendix_motivation}.

\begin{theorem}\label{thm:sensitivity}
    Consider a stratified model \eqref{eq:lap_strat} with zero local regularization, where a strata denoted by $k_0$ has no data sample. Denote the edge weight matrix as $W$. Consider a small perturbation to the graph Laplacian $L=\mathcal{G}(W)$. In particular, for any arbitrarily fixed $i\neq k_0$, we replace $W_{ik_0}>0$ with $\widetilde{W}_{ik_0}:=W_{ik_0}+\epsilon$ for some $\epsilon>0$, while other edges remain the same and denote the perturbed Laplacian by $\widetilde{L}$. Suppose that the solution of \eqref{eq:lap_strat} given $L$ and $\widetilde{L}$ are $\Theta$ and $\widetilde{\Theta}$, respectively, and the true parameter is $\Theta^*$. If for some $\delta, \delta'>0$, both solutions are $(\delta, \delta')$-accurate, in the sense that $\max_{j\neq k_0} \frac{||\theta_j-\theta_j^*||}{||\theta_j^*||} \leq \delta,\ \frac{||\theta_{k_0}-\theta_{k_0}^*||}{||\theta_{k_0}^*||} \leq \delta',\ \max_{j\neq k_0} \frac{||\widetilde{\theta}_j-\theta_j^*||}{||\theta_j^*||} \leq \delta,\ \frac{||\widetilde{\theta}_{k_0}-\theta_{k_0}^*||}{||\theta_{k_0}^*||} \leq \delta' $,  then 
    \begin{equation}\label{eq:sensitivity}
         \frac{\epsilon}{W_{ik_0}} \leq \frac{2(\delta+\delta'+\frac{\delta W_{ik_0}}{S}A_i)}{\left(1-\frac{(1+2\delta)W_{ik_0}}{S}-\delta\right)A_i-1+\frac{W_{ik_0}}{S}-2(\delta+\delta')} \frac{S}{W_{ik_0}},
    \end{equation}
    provided that the denominator is positive. Here $A_i:= \frac{||\theta_i^*||}{\max_{j\neq i} ||\theta_j^*||}, S:=\sum_{j\neq k_0}W_{jk_0}$.
\end{theorem}

\begin{remark}
The above theorem shows that LRSM are sensitive to edge weights if parameter scale and sample size are highly imbalanced. To simplify \eqref{eq:sensitivity}, let $\frac{W_{ik_0}}{S}$ be a positive constant less than $1$ and let $\delta$ close to 0. Then to ensure that both solutions are robust in the way described above, $\frac{\epsilon}{W_{ik_0}}\leq \frac{2\delta'}{(1-\frac{W_{ik_0}}{S})(A_i-1)-2\delta'}\frac{S}{W_{ik_0}}$. If there exists some stratum $i$ such that $||\theta_i^*||$ is relatively large and thus $A_i$ is large, then $\frac{\epsilon}{W_{ik_0}}\lesssim \frac{\delta'}{A_i}$. In another word, to obtain a robust solution in such case, we need a highly accurate Laplacian $L$ and thus traditional methods such as grid search with cross validation is too expensive since a very dense grid is necessary.
\end{remark}


We give a synthetic example to further illustrate our idea. Consider least square linear regression model. The graph has $K$ nodes and on each node $v$, there are $n_k$ (possibly $n_k = 0$) samples $(x_{ki}, y_{ki})\in \reals^n \times \reals$, for $i=1,\cdots,n_k$. The local loss function is square error $l_k(\theta) = \sum_{i=1}^{n_k}(y_{ki}-x_{ki}^T\theta)^2$ and local regularization is sum of squares.

We generate true parameters $\Theta^*$ from multi-variate Gaussian distribution $\mathcal{N}(\mu, \Sigma)$ for each dimension. The inputs $x_{ki} \stackrel{\iid}{\sim} \mathcal{N}(0, I_n)$ and observations are $y_{ki} = x_{ki}^T\theta_k^* + \varepsilon_{ki}$ with $\varepsilon_{ki}\stackrel{\iid}{\sim}\mathcal{N}(0, \sigma^2)$.

We construct a star graph edge weight matrix $W$ with center node denoted by $k_0$ and let $\Sigma = \mathcal{G}(W)^{\dagger} + \delta I_n$. Here $\mathcal{G}(W)^{\dagger}$ is the pseudo inverse of Laplacian $\mathcal{G}(W)$ and $\delta>0$ is small to make $\Sigma$ strictly positive definite. The sample size of the center stratum is small while the rest strata have access to a large number of samples. We generate data as described above, but set the scale of true parameters highly imbalanced intentionally. The edge weights $W^*$ of the optimal star graph are constructed by solving the equation $(\sum_{k\neq k_0}w_{k})\theta_{k_0}^* = \sum_{k\neq k_0}w_{k}\theta_k^*$. Then we gradually add perturbation to $W^*$ and fit the model with perturbed Laplacian $\widetilde{W}=W^*+\epsilon W_1$, where $\widetilde{W}$ is still the edge weight matrix of a star graph. In particular, we let $K=30,\ n=20$. The sample size of all the non-center strata is $100$. $W_1$ is an edge weight matrix of a star graph which is sparse and has $\text{Unif}([0,1])$ distributed entries. Then we re-scale $W_1$ so that it has the same Frobenius norm as $W^*$ does.

We examine two groups that exhibit varying degrees of parameter scale imbalance, as illustrated in Figure \ref{fig:sensitivity norm}. We run 50 random trials of $W_1$ for both groups. Figure \ref{fig:sensitivity star} shows 
the normalized distance between estimated parameter and true parameter at center stratum, \ie, $||\widehat{\theta}_{k_0} - \theta_{k_0}^*|| / ||\theta_{k_0}^*||$ against the perturbation $\epsilon$. The solid line is the mean value and the shaded area is $90\%$ confidence interval. Note that the scale of $W_1$ is the same as that of $W^*$ and thus the curve is approximately a straight line when $\epsilon$ is sufficiently small (intuitively, linearized by first-order Taylor expansions). Although it is only a tiny perturbation, we observe a significant variation in performance, particularly as the parameter scales display more imbalance. The comparison between the left and right subplots in Figures \ref{fig:sensitivity star} and \ref{fig:sensitivity norm} further demonstrates that as the imbalance of parameter scale grows, the sensitivity increases. 

Additionally we increase the sample size of the center strata and show the slope of the normalized distance with respect to perturbation under different sample sizes in Figure \ref{fig:sensitivity slope}. This illustrates that our ideas also hold for few-shot learning.
\begin{figure}[h]
    \centering
    \begin{minipage}[t]{0.45\textwidth}
        \centering
        \includegraphics[scale=0.45]{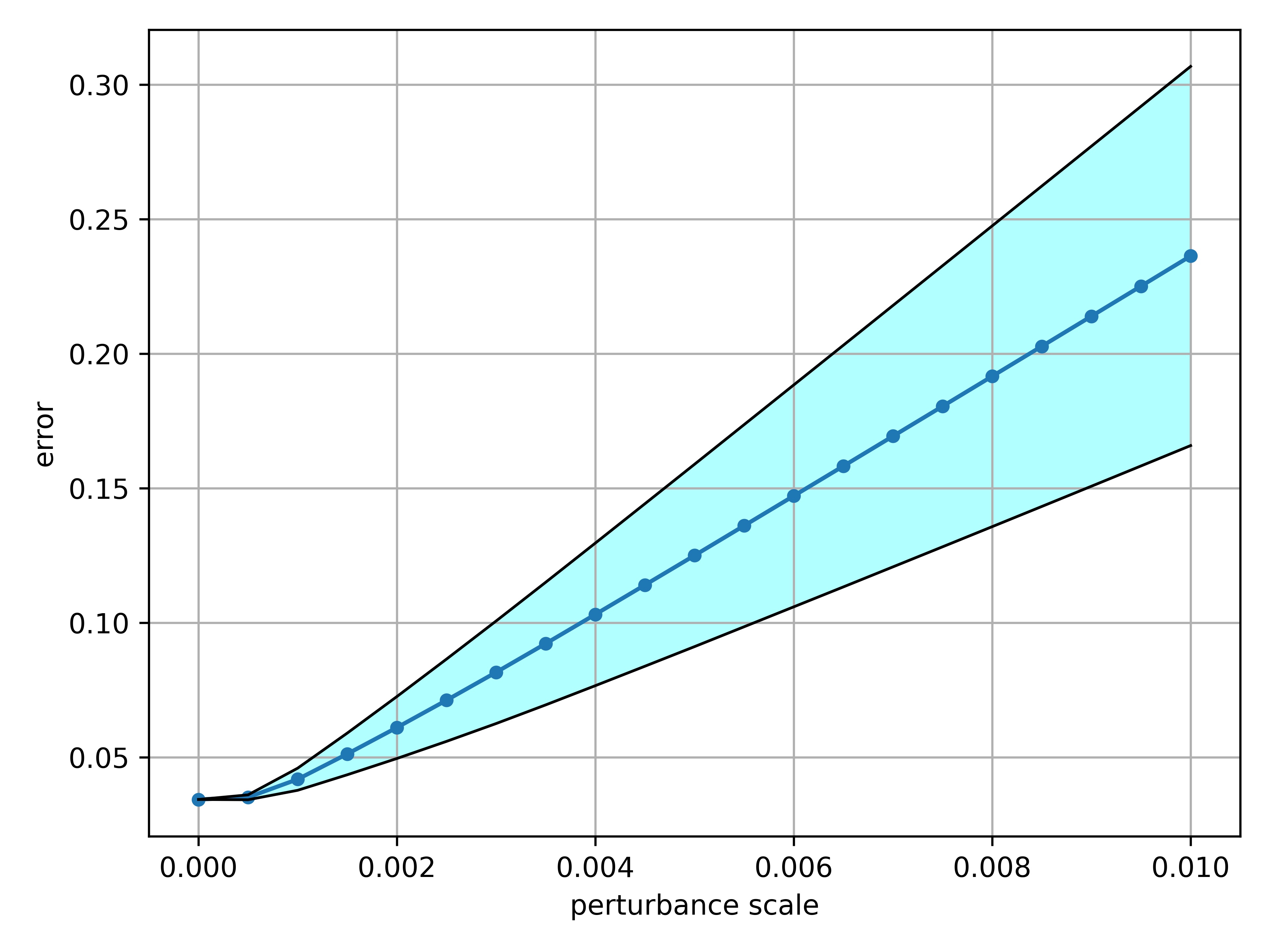}
    \end{minipage}
    \begin{minipage}[t]{0.45\textwidth}
        \centering
        \includegraphics[scale=0.45]{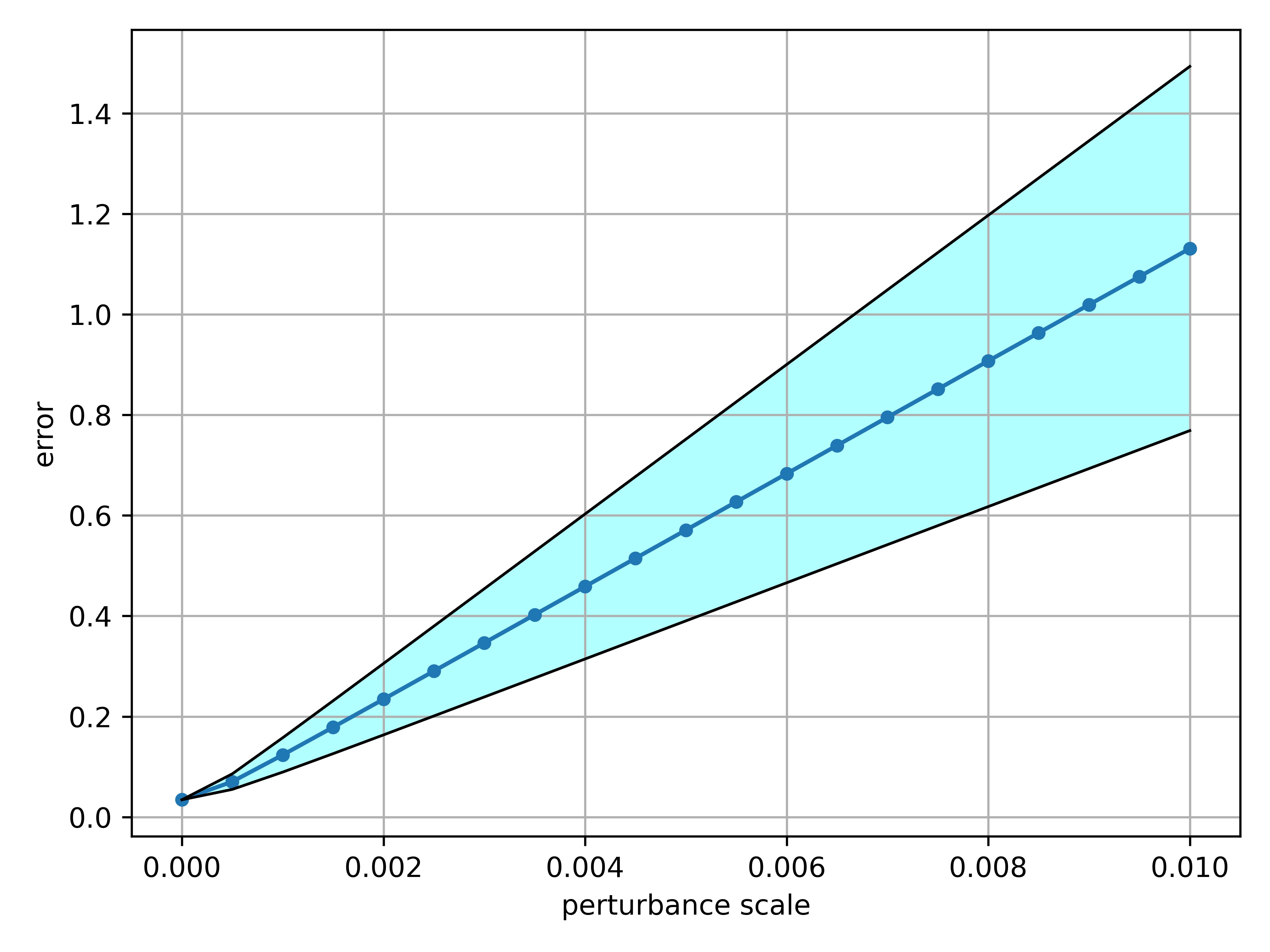}
    \end{minipage}
    \caption{Sensitivity of Laplacian Matrix}
    \label{fig:sensitivity star}
\end{figure}

\begin{figure}[h]
    \centering
    \begin{minipage}[t]{0.45\textwidth}
        \centering
        \includegraphics[scale=0.45]{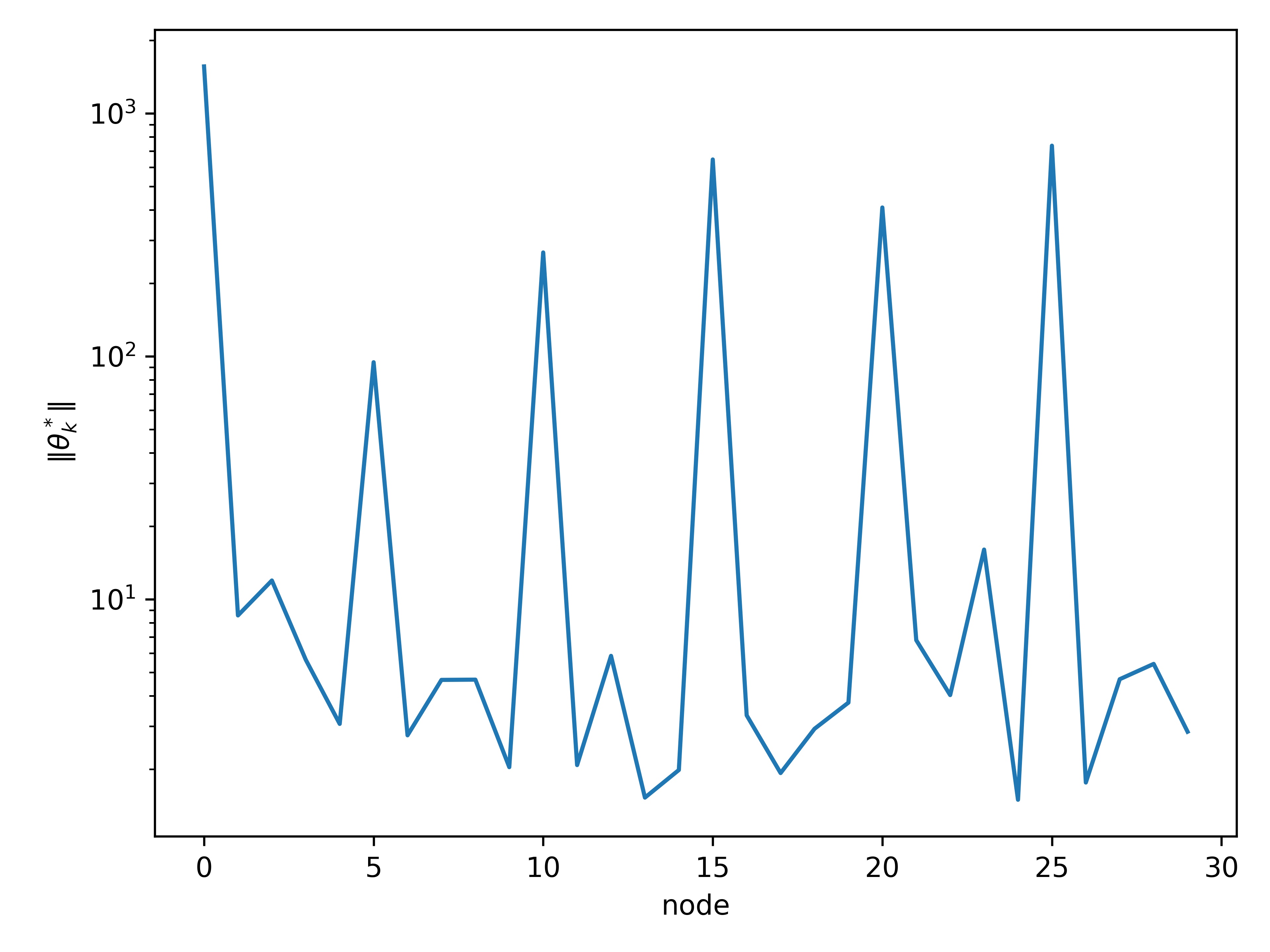}
    \end{minipage}
    \begin{minipage}[t]{0.45\textwidth}
        \centering
        \includegraphics[scale=0.45]{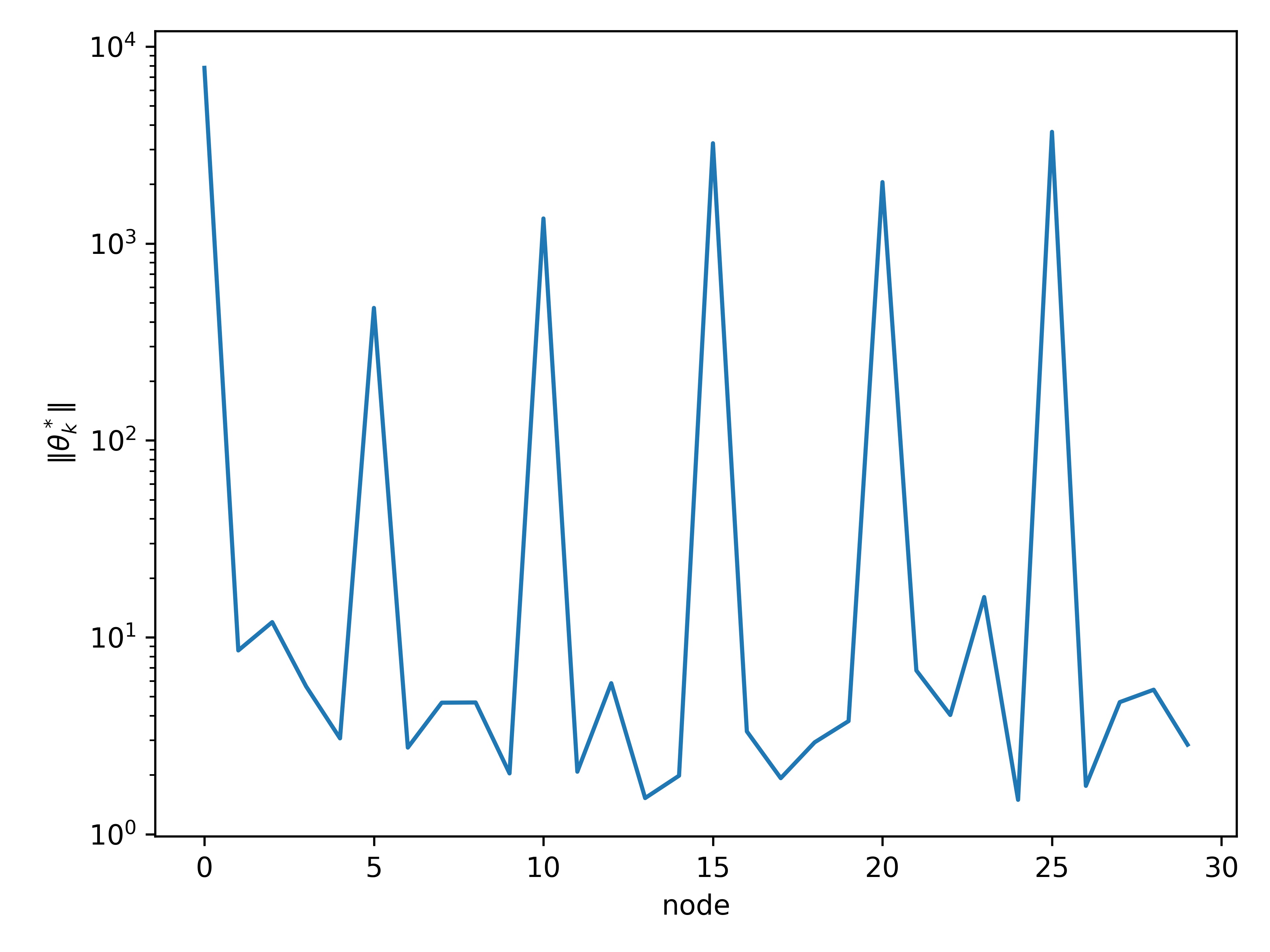}    
    \end{minipage}
    \caption{Norm of true parameters (the last one is center stratum)}
    \label{fig:sensitivity norm}
\end{figure}

\begin{figure}[h]
    \centering
     \begin{minipage}[t]{0.45\textwidth}
        \centering
        \includegraphics[scale=0.5]{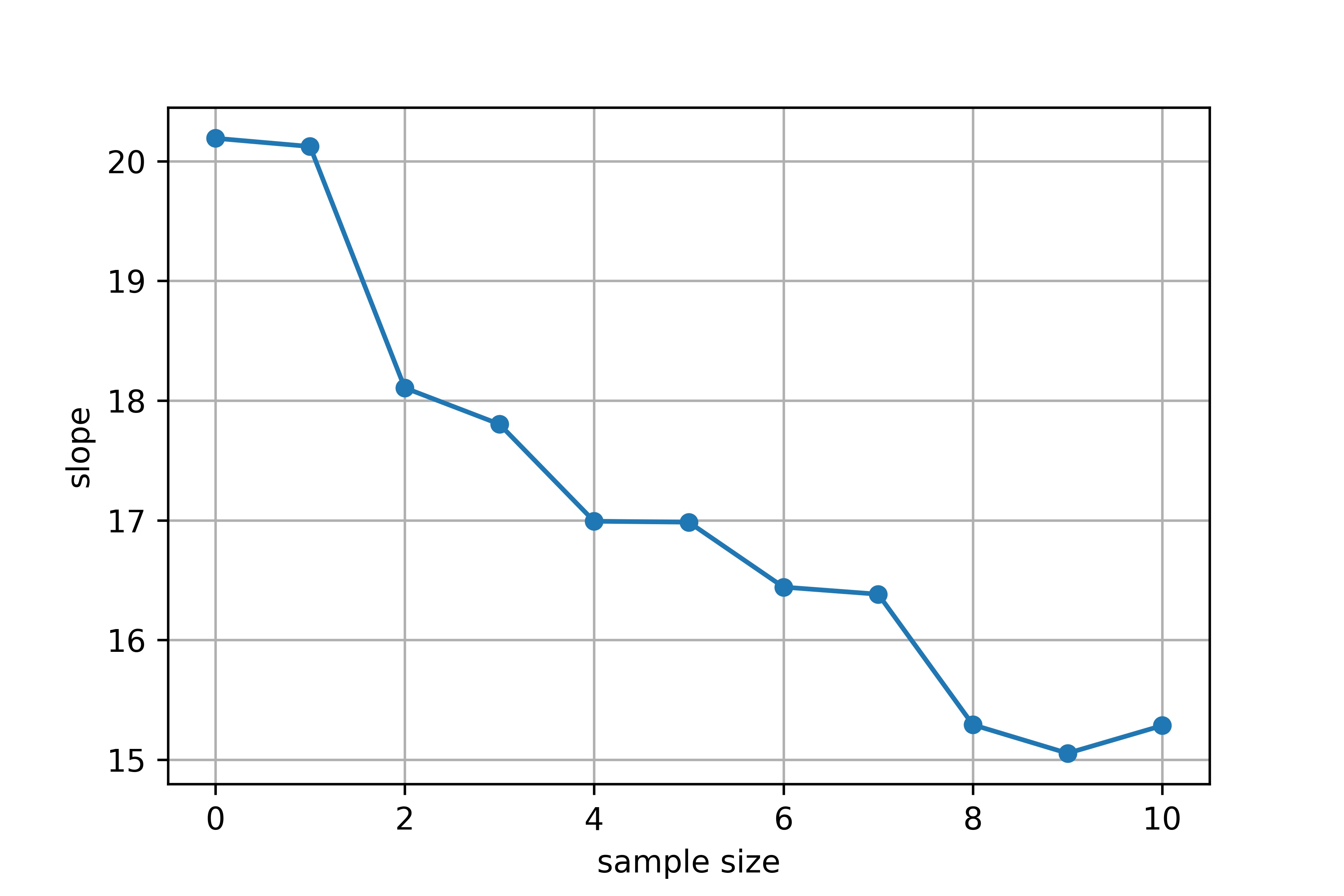}
    \end{minipage}
    \begin{minipage}[t]{0.45\textwidth}
        \centering
        \includegraphics[scale=0.5]{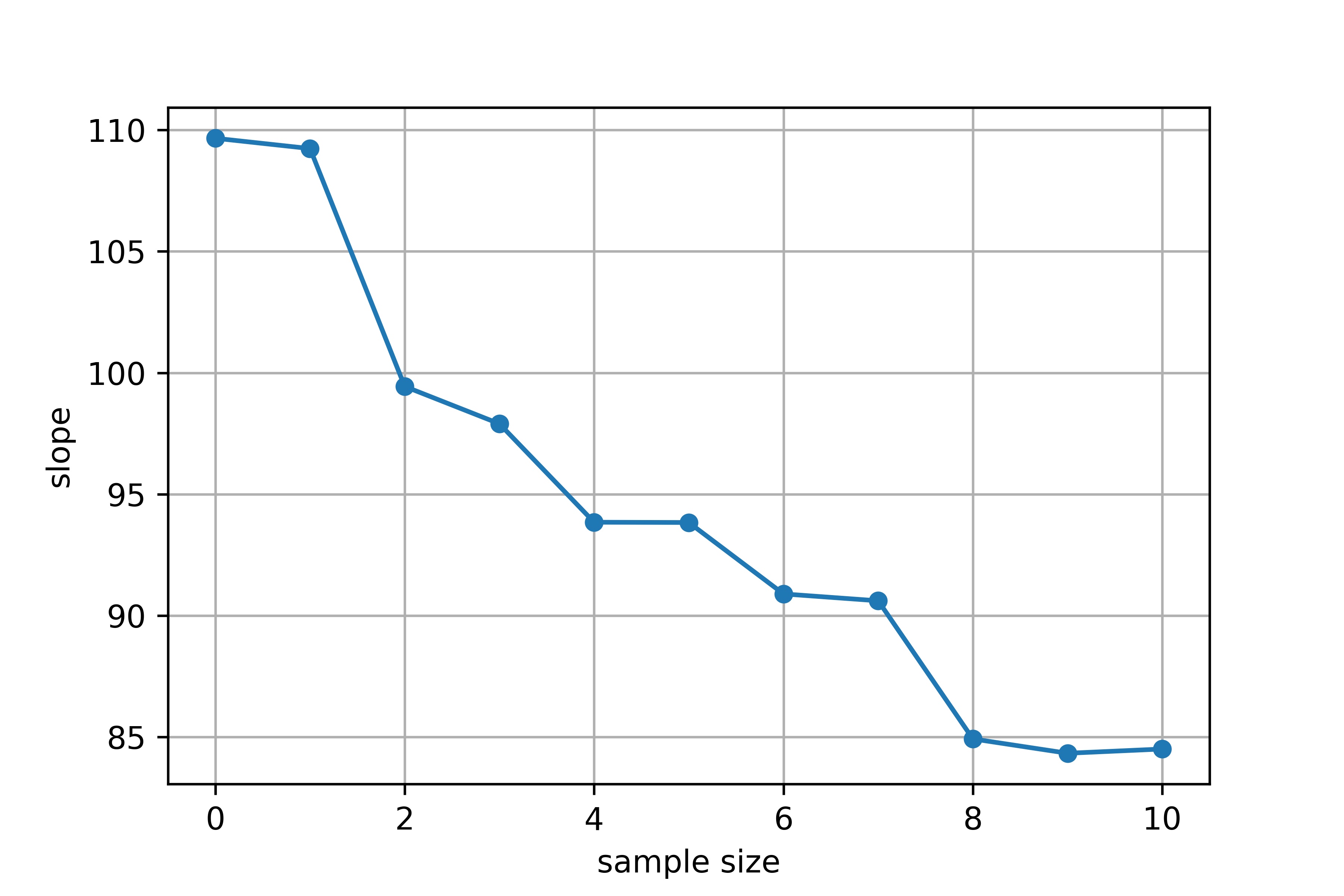}    
    \end{minipage}
    \caption{Slope with different sample sizes}
    \label{fig:sensitivity slope}
\end{figure}

\section{Our method}\label{sec: method}
In order to tackle unknown graph and edge weights, one immediate solution is to jointly minimize $\Theta$ and $W$, with $W\geq 0$ in \eqref{eq:lap_strat}. However, directly doing so will lead to the meaningless solution $W=0$. To address this issue, we need some additional criteria on $W$ that counters such an 
effect of trivial minimization. We first state our proposed new framework as a heuristic and then validate it from both graph-theoretic and statistical perspectives. 

\subsection{Joint Laplacian stratified model}
In this paper, we propose to add two additional terms to \eqref{eq:lap_strat} 
(with $\lambda_1,\,\lambda_2\geq 0$):
\BEQ\label{joint_lap_strat_gen}
\begin{array}{ll}
\text{minimize}_{W,\,\Theta}&\sum\limits_{k=1}^K(l_k(\theta_k)+r(\theta_k))
+\frac{1}{2}\sum\limits_{i=1}^K \sum\limits_{i < j} W_{ij}\|\theta_i-\theta_j\|_2^2+\lambda_1 R_1(W)
+\lambda_2 R_2(W),\\
\text{subject to}& W\geq 0,\,\textbf{diag}(W)=0,\,W\in\symm^K.
\end{array}
\EEQ
Here the first term
\BEQ\label{logdet-reg}
R_1(W)=-\log\det(\mu I+\mathcal{G}(W))
\EEQ
is the negative log determinant of the regularized Laplacian matrix $\mathcal{G}(W)$, which 
encourages accessibility, connectivity and robustness in the graph. Alternatively, it can be easily derived from a probabilistic model.  See \S \ref{interpretation} for more details. Besides, it prevents $W$ from trivially degrading to zero and thus enlarges the magnitude.
Here $\mu>0$ is a positive constant.

The second term (with $\eta\in[0,1]$)
\BEQ\label{prox-reg}
R_2(W)=\dfrac{1-\eta}{2}\|W-W_0\|_F^2+\eta \|W\|_1 
\EEQ
is a elastic net style regularization term on $W$, consisting of a proximal term enforcing its closeness to 
some prior guess $W_0\in\reals_+^{K\times K}$ and an auxiliary penalization term  
encouraging edge sparsity in the similarity graph.
Here $\|\cdot\|_1$ is the elementwise $\ell_1$ norm. The idea of using prior guess $W_0$ is also adopted in \cite{Zhao2021JointAG}. Throughout this paper we suppose that $W_0$ is diagonal-free.

\paragraph{Example Choices of $W_0$.} We can simply set $W_0$ to be zero matrix or based on a combination of prior info, hand-made design and cross validation as is done in \cite{tuck2019distributed}, which are both very effective in practice as we will see in our numerical experiments below. Additionally, below we describe three slightly more advanced ways to obtain a non-trivial (and potentially better) initial guess $W_0$. For simplicity and since the naive approaches already work well in practice as is shown in our experiments below, these more advanced methods of choosing $W_0$ are not numerically tested in our experiments and we leave the investigation of their strengths and weaknesses for future work. 

The first one relies on certain metrics over the stratification 
variable $z$. For example, when each value of $z\in\{1,\dots,K\}$ is naturally associated 
with a representation 
$f_z\in\mathcal{M}$ for some metric space $\mathcal{M}$ with metric $d$, 
one can set $[W_0]_{ij}=w(f_i,f_j)$, where
$w:\mathcal{M}\times\mathcal{M}\rightarrow\reals_+$ 
is some ``kernel'' function, \ie, $w(f_i,f_j)$ is 
larger for when $f_i$ and $f_j$ are closer under the metric $d$.
The following are common examples of $w$:
\begin{itemize}
\item Exponential kernel: $w(f_i,f_j)=\exp(-\tau d(f_i,f_j))$, with $\tau>0$ being a positive 
constant. 
\item $k$-nearest neighbor kernel \cite{ziko2020laplacian}: 
$w(f_i,f_j)=1$ if $f_j$ is within the $k$-nearest neighborhood (in terms of $d$) 
of $f_i$, and $w(f_i,f_j)=0$ otherwise.
\end{itemize}
See also \cite{tuck2021portfolio} for designing $w$ in the specific setting 
of portfolio construction.

The second choice of $W_0$ is to first solve all the sub-models 
independently and obtain individually learned parameters 
$\tilde{\theta}_i=\argmin_{\theta}\,l_i(\theta)+r(\theta)$ for each sub-model $i=1,\dots,K$, and 
then define $[W_0]_{ij}=w(\tilde{\theta}_i,\tilde{\theta}_j)$, where 
$w:\reals^n\times\reals^n\rightarrow\reals_+$ is again some ``kernel'' function. The examples of
$w$ above also applies here, with the metric $d(\tilde{\theta}_i,\tilde{\theta}_j)=
\|\tilde{\theta}_i-\tilde{\theta}_j\|_2$. This approach can also be seen as a generic approach
for finding representations $f_z$ (when they are not directly available) for the stratification 
variable $z$. 
Note that one potential limitation of this method is that when some stratification value $z$ 
has very few or even no data points, the learned parameter $\tilde{\theta}_z$ 
could be meaningless. 

The third approach is based on the idea of auto-tuning. In this approach, a validation set of 
data points $\{z_i^{\rm val},x_i^{\rm val},y_i^{\rm val}\}_{i=1}^{N_{\rm val}}$ is held out to 
construct the following validation loss function:
\BEQ\label{val_loss}
L^{\rm val}(\Theta)=\sum\nolimits_{k=1}^Kl_k^{\rm val}(\theta_k),
\EEQ
where $l_k^{\rm val}(\theta)=\sum_{i:z_i^{\rm val}=k}l(\theta,x_i^{\rm val},y_i^{\rm val})$. 
Let $\mathcal{S}:\reals^{K\times K}\rightarrow\reals^{n\times K}$ be a mapping from the 
graph weights $W$ to a solution $\Theta$ of the training problem \eqref{eq:lap_strat}. Then 
we calculate $W_0$ as the solution to the following problem:
\BEQ\label{W0_val_loss}
\begin{array}{ll}
\text{minimize}_{W} & L^{\rm val}(\mathcal{S}(W)),\\
\text{subject to} & W\geq 0,\,\textbf{diag}(W)=0,\,W\in{\bf S}^K, 
\end{array}
\EEQ
which minimizes the validation loss over all possible choices of nonnegative graph weights. 
Problem \eqref{W0_val_loss} can be minimized by gradient descent methods, where the gradients
are evaluated by differentiating through convex programs \cite{agrawal2019differentiable}. 
Note that a similar method has been proposed 
for least-squares estimations in \cite{barratt2020least}. It is also tempting to directly use the obtained $W_0$ from this approach as the final choice of the graph weights. However, for general (non least-squares) problems, the optimization of \eqref{eq:lap_strat} is already computationally intensive except for small problems, and we have found that it is not working very well in general when we directly use it as the final weights without further modifications. 


\subsection{Graph theory and statistical interpretations}\label{interpretation}
In this subsection, we provide interpretations of the objective 
\eqref{joint_lap_strat_gen} from the lens of graph theory and statistics. 
We first provide two explanations of the additional terms 
related to the graph weights $W$ in \eqref{joint_lap_strat_gen}, 
and then give an end-to-end interpretation of 
\eqref{joint_lap_strat_gen} as the Maximum a Posteriori estimation of a Bayesian model.

\paragraph{Learning $W$ as spanning forests total weight maximization.} We first explain how
the $\log\det$ term \eqref{logdet-reg} is related to the total weight of spanning forests 
in the weighted graph $G_W$
associated with $W$. Suppose that the graph has $m$ connected components, each containing
$n_i$ $(i=1,\dots,m)$ nodes. Then the zero eigenvalue of the graph Laplacian matrix $\mathcal{G}(W)$ 
has its multiplicity equal to $m$. Suppose that the remaining $K-m$ nonzero eigenvalues are 
$0<\lambda_1\leq \cdots\leq \lambda_{K-m}$. 
Then the 
 weighted matrix tree/forest theorem \cite{chaiken1978matrix, chebotarev2006matrix, 
 mat_tree_lec} states that
 \BEQ\label{weighted_mat_tree}
 \tau(G_W) = \dfrac{1}{\prod_{i=1}^mn_i}\lambda_1\cdots\lambda_{K-m},
 \EEQ
where $\tau(G_W)$ is the total weight of spanning forests in the graph $G_W$, defined as
\[
\tau(G_W)=\prod\nolimits_{i=1}^m\sum\nolimits_{T\in\mathcal{T}_i}w(T).
\]
Here $\mathcal{T}_i$ denotes the set of all spanning trees in component $i$, and 
$w(T)$ denotes the weight of tree $T$, defined as $w(T)=\prod_{(i,j)\in T}W_{ij}$, \ie, 
 the product of weights for all edges in $T$. 
 
 Hence we see that as $\mu\rightarrow0$, we have
 \[
 \begin{split}
 \log\det(\mu I+\mathcal{G}(W))&=\log \left(\mu^m\prod\nolimits_{j=1}^{K-m}(\mu+\lambda_j)\right)\\
 &=m\log \mu+\log\tau(G_W)+\sum_{i=1}^m\log n_i+O(\mu).
 \end{split}
 \]
 This indicates that for a small regularization parameter $\mu\in(0,1)$, 
 minimizing $-\log\det(\mu I+\mathcal{G}(W))$ is 
 approximately achieving three goals simultaneously, namely 
 maximizing the total weight of spanning forests (the $\log \tau(G_W)$ term), minimizing
 the total number of connected components (the $-m\log(1/\mu)$ term) and keeping the partition of
 the number of nodes as even as possible across different connected components 
 (the $\sum_{i=1}^m\log n_i$ term). Here larger abundance of spanning forests, 
 smaller number of connected components and more even partitioning of nodes 
 all encourage better connectivity and accessibility of the graph $G_W$ 
 \cite{chebotarev2006matrix_social}. This 
  leads to a more robust and comprehensive description of the relationship 
 among the $K$ sub-models, which is desired.

\paragraph{Hierarchical Gaussian Markov random field.} A GMRF \cite{Rue2005GaussianMR} based on an undirected graph is defined as 
\begin{equation*}
    p(\Theta\vert W) \propto \exp{(-E(\Theta))},
\end{equation*}
with energy function $E(\Theta) := \frac{1}{2}\sum_{i<j} W_{ij} \Vert \theta_i-\theta_j\Vert^2 + \frac{\mu}{2}\Vert \Theta\Vert^2$. Rewrite it and we can get the normalized density 
\begin{equation*}
    p(\Theta\vert W) = (2\pi)^{-\frac{nK}{2}}\det{(\mu I + \mathcal{G}(W)})^{n/2}\exp{(-\frac{1}{2}Tr(\Theta (\mu I + L) \Theta^T))}.
\end{equation*}
And we can also have a prior distribution on $W$ such that 
\begin{equation*}
    p(W) \propto \exp{(-\frac{c_1}{2}\Vert W\Vert_F^2 - c_2\Vert W\Vert_1)},
\end{equation*}
which is a mixture of Gaussian and Laplacian distribution. Furthermore, consider a hierarchical Bayesian model with likelihood $p(y\vert x, z, \Theta) \propto \exp{(-l(y,x,z;\Theta))}$. Then the negative log-likelihood of the whole model is (up to an additive constant):
\begin{equation*}
    -\log{p} = l(y,x,z;\Theta) + \frac{\mu}{2} \Vert\Theta\Vert^2 + \mathcal{L}(\Theta) - \frac{n}{2}\log{\det (\mu I + \mathcal{G}(W))} + \frac{c_1}{2}\Vert W\Vert_F^2 +c_2\Vert W\Vert_1.
\end{equation*}
By setting $\lambda_1=\frac{n}{2},\ r(\theta)=\frac{\mu}{2}\Vert \theta \Vert^2,\lambda_2 = c_1+c_2, \eta = \frac{c_2}{c_1+c_2} $, we find the Maximum a Posterior (MAP) estimation of $(\Theta, W)$ is exactly given by (\ref{joint_lap_strat_gen}). Note that if we set $\mu = 0$ and define the pseudo-determinant as the product of all non-zero eigenvalues of $\mathcal{G(W)}$, then we recover the improper GMRF as in \cite{Kumar2019AUF}.

\section{Comparison with existing work}
Before we move on to solving our proposed Joint Laplacian stratified model, in this section, we first review and summarize the related work and compare our approach with them.

\paragraph{Laplacian regularized stratified model fitting.} Laplacian regularized stratified model fitting with a known graph has been studied for years \cite{ando2006learning,sheldon2008graphical}, and there is a recent surge of interest in improving its scalability \cite{tuck2020eigen} as well as robustness and adaptivity \cite{Duan2022AdaptiveAR,Lam2022AdaptiveDF}. There are also many recent works applying these models to  semi-supervised learning \cite{Slepev2017AnalysisO, Cabannes2020OvercomingTC}, 
joint covariance estimation \cite{tuck2020fitting}, few-shot learning \cite{ziko2020laplacian}, portfolio 
construction with multiple market conditions \cite{tuck2021portfolio} and federated learning \cite{Dinh2021ANL}, to name just a few. 

\paragraph{Multi-task learning.} When the strata are interpreted as tasks, stratified models can also be seen as multi-task learning. 
However, to build up connection between tasks based, the literature of multi-task learning has been mainly focused on regularization terms that are not Laplacian regularization \cite{jacob2008clustered,zhou2011clustered,gonccalves2016multi,Smith2017FederatedML, Evgeniou2004RegularizedML, Zhang2010ACF}, which are typically 
less explicit and interpretable than graph Laplacian/weights and are beyond the scope of this paper. 

\paragraph{Graph learning.} Graph learning from data is a canonical 
problem that has received substantial attention in the literature. Numerous studies have investigated approaches to learning a structured graph while utilizing various forms of prior knowledge, including spectral constraints, sparsity, node degrees, and edge patterns \cite{Kalofolias2017LargeSG, Kumar2019AUF, Kang2021StructuredGL}. Also we can incorporate prior knowledge along nodes with that on observation side as in \cite{Pu2020KernelBasedGL}. Structured information is frequently available in many real-word applications including gene network analysis, community detection, clustering. However, practitioners should use specific algorithms to tackle with different types of constraints. 

When there is lack of knowledge on graph structure, some general methods to learn the graph from data (or signals) have been proposed. In scenarios where there is limited knowledge about the graph structure, several general methods have been proposed for graph learning from data or signals. In the literature, these methods can be broadly categorized into three classes, all of which involve minimizing the sum of Laplacian and a regularization function of $W$. First, \cite{Dong2014LearningLM} considers solving
\begin{equation}\label{eq: TC}
    \begin{array}{ll}
        \text{minimize}_{L} & R_{\text{TC}}(L) := Tr(\Theta L \Theta^T) + \frac{C_2}{2} \Vert L\Vert^2 \\
        \text{subject to} & Tr(L) = C_1, L_{ij} = L_{ji} \leq 0, L\cdot\bm{1} = 0.
    \end{array}
\end{equation}
We refer to (\ref{eq: TC}) as the Tr-Constraint method. It imposes a hard constraint on the total edge weight directly to avoid trivial solution. But using a Frobenius norm on the Laplacian matrix is confusing: the elements of $L$ have different scales and are linearly dependent. The second formulation is to solve
\begin{equation}\label{eq: LD}
    \begin{array}{ll}
        \text{minimize}_{W} & R_{\text{LD}}(W) := Tr(\Theta \mathcal{G}(W) \Theta^T) -C_1 \bm{1}^T\log{(W\cdot\bm{1})} + \frac{C_2}{2} \Vert W\Vert^2 \\
        \text{subject to} & W_{ij} = W_{ji} \geq 0, W_{ii}=0.
    \end{array}
\end{equation}
We refer to (\ref{eq: LD}) as Log-Diagonal method. The logarithmic barrier acting on the node degree vector $W\cdot\bm{1}$ can prevent degradation and improve overall connectivity of the graph, without compromising sparsity \cite{Kalofolias2016HowTL}. However, the logarithmic barrier strictly makes every node has positive degree and thus prevents isolated node. This property is not generally desirable since it is highly possible that certain node indeed has no or even negative correlation with other nodes and it should be isolated in the graph. Such a problem is even more serious for the third class of approach based on entropy:
\begin{equation}\label{eq: Ent}
    \begin{array}{ll}
        \text{minimize}_{W} & R_{\text{Ent}}(W) := Tr(\Theta \mathcal{G}(W) \Theta^T) + \sigma^2 \sum_{i\neq j} W_{ij}(\log{W_{ij}}-1) \\
        \text{subject to} & W_{ij} = W_{ji} \geq 0, W_{ii}=0,
    \end{array}
\end{equation}
which has a closed-form solution $W_{ij} = \exp{(-\frac{\Vert \theta_i-\theta_j\Vert^2}{2\sigma^2})}$. We refer to (\ref{eq: Ent}) as Entropy method. The graph then must be fully connected and cause negative effects on performance. This method is not preferred in general compared with the other two mentioned above \cite{Kalofolias2016HowTL}, but it can serve as a good initial guess (see \S \ref{sec: method}).

Most of the methods in the literature about graph learning with minimum assumptions only focus on degree properties of nodes, neglecting the interaction between edges of a single node. However, assigning weights to each edge, given a node's degree, is a crucial yet intricate process as shown in \S \ref{sec: method}.

\paragraph{Joint graph learning with Laplacian Regularization.} To jointly learn both the model and the graph is not a new idea in the community and there are some recent works about it for specific application. Most of them combine the specific task with the two general graph learning methods, Tr-Constraint and Log-Diagonal, mentioned above, or with some modifications tailored to the task. \cite{VargasVieyra2020JointLO} considers representation learning for graph-based semi-supervised learning. It jointly learns the representation and the underlying graph for use in the downstream graph-based semi-supervised learning tasks. It adopts the Log-Diagonal method to learn the graph, and the goal is such that the representation encodes the label information injected into the graph, while the graph provides a smooth topology with respect to the transformed data. 
\cite{He2022StockPW} considers financial applications and assumes that the graph is $k$-clustered. Hence it adopts the Tr-Constraint method with an additional low rank constraint on the Laplacian matrix, which leads to a much more difficult optimization problem w.r.t. the graph weights. \cite{Zhao2021JointAG} uses a similar method and aims at unsupervised feature selection.  
It also leverages an existing similarity graph matrix constructed from original high dimensional data as a prior. \cite{Karaaslanli2022SimultaneousGS, Li2016JointMC} consider graph clustering given heterogeneous data. They use  the Tr-Constraint method to learn graph topology while learning the labels (communities) of given data.

\paragraph{Comparison with existing methods.} The most significant distinction in our approach is the usage of a regularized log-determinant term, which offers better interpretability and overcomes limitations associated with existing methods. Notably, it enables modeling of unconnected graphs and even isolated nodes, thus improving the expressive capacity of our models and the potential to learn a broader range of graph structures compared to Log-Diagonal and Entropy methods. For instance, under circumstances where a stratum differs significantly from others, our approach allows more flexibility to model it as an isolated node, improving interpretability and accuracy. Additionally, the log-determinant term considers the interaction between each edge and their weights, instead of using information from a single edge or a single node degree as in Tr-Constraint, Log-Diagonal and Entropy method. This will help to characterize almost all essential information of a graph based on the spectral graph theory.


\section{Optimization: Monotone accelerated proximal gradient method}\label{sec:opt}
Note that when assuming all the $l_k(\cdot)$ and $r(\cdot)$ are convex, the training objective is biconvex with respect to $\Theta$ and $W$ and it is natural to consider alternating minimization method. However, it is difficult to efficiently solve the sub-problem given $\Theta$ or $W$. With this concern, we rewrite (\ref{eq:lap_strat}) as following:
\begin{equation}
    \begin{array}{ll}
         \min F(\Theta,W) := f(\Theta, W)+g(\Theta, W)
    \end{array}
\end{equation}
where
\begin{equation}
    \begin{split}
        g(\Theta, W)&=\sum_k r(\theta_k)+\lambda_2\eta\Vert W\Vert_{1}+\sum_{i<j}\mathcal{I}_{W_{ij}=W_{ji}\geq 0}+\mathcal{I}_{\textbf{diag} (W)=0}\\
        f(\Theta, W)&=\sum_k l_k(\theta_k)+\frac{1}{2}\sum_{i<j}W_{ij}\Vert \theta_i - \theta_j\Vert_{2}^{2}+\frac{\lambda_2(1-\eta)}{2}\Vert W-W_0 \Vert_{F}^{2}\\
        &\quad -\lambda_1 \log{\det{(\mu I+\mathcal{G}(W))}}
    \end{split}
\end{equation}

We take $f$ as a smooth term and $g$ as a simple regularizer with proximal oracle. Overall, it is a nonconvex, nonsmooth composite optimization problem due to the Laplacian regularization. Here we adopt the MAPG algorithm, which is first proposed in \cite{NIPS2015_f7664060}.

\subsection{Algorithm}
Let $\bm{x}=(\Theta,W)$, the procedure is shown in Algorithm \ref{alg_APG}. The gradient of $l_k$ for $k=1,\cdots,K$ can be computed in parallel and thus this algorithm is still in a distributed fashion, in spite of coupling of the variables in the Laplacian regularization term.

\begin{algorithm}[ht]
\caption{\textbf{M}onotone \textbf{A}ccelerated \textbf{P}roximal \textbf{G}radient}
\label{alg_APG}
\begin{algorithmic}
\STATE{
Initialize $\bm{z}_1=\bm{x}_1=\bm{x}_0,t_1=1,t_0=0, \alpha_x,\alpha_y>0$.
}
\FOR{$k=1,\cdots$}
\STATE{
\begin{equation*}
\begin{aligned}
    &\bm{y}_k=\bm{x}_k+\frac{t_{k-1}}{t_k}(\bm{z_k}-\bm{x}_k)+\frac{t_{k-1}-1}{t_k}(\bm{x}_k-\bm{x}_{k-1}),\\
    &\bm{z}_{k+1}=\textbf{prox}_{\alpha_y g}(\bm{y}_k-\alpha_y \nabla f(\bm{y}_k)),\\
    &\bm{v}_{k+1}=\textbf{prox}_{\alpha_x g}(\bm{x}_k-\alpha_x \nabla f(\bm{x}_k)),\\
    &t_{k+1}=\frac{\sqrt{4(t_k)^2+1}+1}{2},\\
    &\bm{x}_{k+1}=\left\{\begin{array}{cl}
         \bm{z}_{k+1},& \text{if}\ F(\bm{z}_{k+1})\leq F(\bm{v}_{k+1}),\\
         \bm{v}_{k+1},& otherwise 
    \end{array}\right.
\end{aligned}
\end{equation*}
}
\ENDFOR
\end{algorithmic}
\end{algorithm}

\paragraph{Evaluate the proximal operator of $g$.} Notice that $g(\Theta,W)=\sum_{k=1}^K r(\theta_k)+g_2(W)$, where 
\[
g_2(W)=\lambda_2\eta\|W\|_1+\sum_{i<j}\mathcal{I}_{W_{ij}=W_{ji}\geq0}+\mathcal{I}_{\textbf{diag}(W)=0}.
\]
By the basic properties of proximal operators of separable sum functions \cite[\S 2.1]{parikh2014proximal}, we have for any $\alpha\geq 0$, 
\[
\textbf{prox}_{\alpha g}(\bm{x})=(\textbf{prox}_{\alpha r}(\theta_1),\dots,\textbf{prox}_{\alpha r}(\theta_K), \textbf{prox}_{\alpha g_2}(W)).
\]
The proximal operator of $r(\theta_k)$ is closed-form as is assumed at the beginning of the paper, while  
the proximal operator of $g_2(W)$ is given by 
\begin{equation*}
\begin{aligned}
    \textbf{prox}_{\alpha g_2}(U)
    &=\mathop{\arg\min}\limits_{W} \lambda_2\eta\Vert W\Vert_{1}+\sum_{i,j}\mathcal{I}_{W_{ij}=W_{ji}\geq 0}+\mathcal{I}_{\textbf{diag} (W)=0}+\frac{1}{2\alpha}\Vert W-U\Vert_{F}^{2} \\
    &=\max{\left(0,(U+U^\top)/2-\lambda_2\alpha\eta\right)},
\end{aligned}
\end{equation*}
where the max is taken element-wisely and $U$ is diagonal-free. See Appendix \ref{appendix_notes} for more details on the derivation of $\textbf{prox}_{\alpha g_2}$. 

\paragraph{Complexity.} 
The proximal operator of $g(\bm{x})$ often has a closed-form solution and can be computed in parallel. The complexity is dominant by computing the gradient of $f(\bm{x})$, which includes computing the inverse of the Laplacian matrix $\mathcal{G}(W)\in \mathbf{S}^K$, and also by computing the log-determinant in the monotonicity-checking step. In practice, $K$ is generally not too large and several hundreds is preferred.

\paragraph{Initialization.}
Since the problem is nonconvex, different initializations may lead to different solutions. Here we heuristically propose to initialize $W$ as $W_0$, which does well in our experiments. As for the initial point of $\Theta$, we set it to be zero vector.

\paragraph{Stepsizes.}
For simplicity, we choose fixed stepsizes $\alpha_x, \alpha_y$ by grid search in all our numerical experiments. \cite{NIPS2015_f7664060} proposes a line search method with Barzilai-Borwein initialization to adaptively adjust the stepsizes. Note that for our formulation, the cost of computing the value $F(\bm{x})$ is very expensive since there is a log-determinant term, which has complexity of $O(K^3)$. Since line search method possibly needs to compute the function value for many times in a single iteration, it is not preferred here.

\paragraph{Stopping Criterion.}
If we are blessed with a good prior $W_0$, then $W$ should not vary too much from $W_0$ and our algorithm is easy to converge. In this case, our stopping criterion is $\Vert \bm{x}_k-\bm{x}_{k-1}\Vert < \epsilon_{tol}$.

However, in cases where our knowledge is limited or $W_0$ is significantly different from the optimal $W$, to keep training until convergence is time consuming and unnecessary. Consequently we adopt the idea of early stopping. In order to obtain better generalization performance, we test on validation set at each iteration and stop training as long as the validation score has not been improved for a specific number of iterations.

Note that for the ADMM-based algorithm to fit LRSM in \cite{tuck2019distributed}, they adopt a stopping criterion  based on the KKT condition, which is slightly different from ours. For fairness, in all the experiments in \S \ref{sec_exp}, we run their algorithm for sufficiently many iterations to ensure that it has converged and makes no improvement on performance.

\subsection{Convergence analysis}

 Since the Laplacian regularization term is not globally strongly smooth, which is a key assumption in the original proof in \cite{NIPS2015_f7664060}, we need new techniques to ensure that the iteration sequence is bounded. To be specific, we make the following assumptions.
 
\begin{assumption}
\label{asp_apg_l_smooth}
For any $k\leq K$, $l_k$ is proper and strongly smooth in any bounded sets.
\end{assumption}

\begin{assumption}
\label{asp_apg_r_lsc}
$r$ is proper, lower semicontinuous, coercive, and bounded from below (no need to be convex).
\end{assumption}

\begin{assumption}
\label{asp_apg_coercive}
For any $k\leq K$, $l_k(\theta)+r(\theta)$ is coercive, \ie, $l_k(\theta)+r(\theta)\rightarrow +\infty$ as $\Vert\theta\Vert \rightarrow\infty$.
\end{assumption}

Assumption \ref{asp_apg_l_smooth} can ensure the smooth term $f$ is locally strongly smooth in any bounded set and is generally weaker than most of that in the literature since we do not require globally strongly smoothness. Assumption \ref{asp_apg_coercive} is only for technical reason and is also required in \cite{NIPS2015_f7664060}. Typical examples of $l_k$ are regression and classification model as mentioned in \cite{tuck2019distributed}. Examples of $r$ include bounded constraint, sum of squares, $\ell_p$ with $p>0$, elastic net, OSCAR, equality or inequality constraints that are violated at infinity (\eg\ Stiefel manifold, sphere, polytope).



\begin{theorem}
\label{thm_joint_converge}
Under Assumption \ref{asp_apg_l_smooth}-\ref{asp_apg_coercive}, there exist $\alpha_x,\alpha_y>0$ in Algorithm \ref{alg_APG} such that
\begin{itemize}
    \item[(\romannumeral1)] The sequence $\{\bm{x}_k\}$ is bounded.
    \item[(\romannumeral2)] The set of limit points of $\{\bm{x}_k\}$ is nonempty. If $\bm{x}^*$ is a limit point of $\{\bm{x}_k\}$, then $\bm{x}^*$ is a critical point of $F$.
    \item[(\romannumeral3)] If $F(\bm{x})$ satisfies KL property, and the desingularising function has the form of $\phi(t)=\frac{C}{s}t^{s}$ for some $C>0, s\in (0,1]$, then $\{F(\bm{x}_k)\}$ converges to some $F(\bm{x}^*)$ with a linear or sub-linear rate (depending on $s$).
\end{itemize}
\end{theorem}

\begin{proof}[Proof Sketch]
    We first prove by induction that the sequence $\{\bm{x}_k\}$ and $\{\bm{v}_k\}$ are bounded. This is ensured by monotonicity of function value of our algorithms given a sufficiently small stepsize. Then by Assumption \ref{asp_apg_l_smooth} strongly smoothness property holds at iteration sequence. The remaining part is the same as in \cite{NIPS2015_f7664060}. We defer the complete proof in Appendix \ref{appendix_proof}.
\end{proof}

\section{Numerical results}
\label{sec_exp}


We take LRSM as a baseline and show the advantages of our joint graph learning method. This section is divided into two parts, one to show improvement over good prior $W_0$, the other to demonstrate flexibility when there is no prior (\ie, $W_0=0$). We also include \textit{common model} (one that does not depend on $z$) and \textit{separate model} ($W_{ij}\equiv 0$) as two extremes. The code that implements our proposed methodology as well as the baselines to reproduce the experiments below are available online at \url{https://github.com/cvxgrp/joint-lrsm}. 

\subsection{Stratified model with prior}
In this subsection, we consider two examples in \cite{tuck2019distributed} and take the hand-designed graphs there as the prior $W_0$. Our goal is to learn better $W$ given the prior info $W_0$ through our Joint Laplacian stratified model. The results show that our method can improve the performance by a notable margin.

\subsubsection{Mesothelioma classification}

We consider the problem of predicting whether a patient has mesothelioma, a form of cancer, given their sex, age, and other medical features that were gathered during a series of patient encounters and laboratory studies. For each class, we use logistic regression to model the conditional probability of contracting mesothelioma given the features. There are 324 samples in total and we split 10\% of them as test set.

Data pre-processing as well as the setting of the model and the baseline algorithms are the same as \cite{tuck2019distributed}, except that we re-tune the local regularization weight to get better performance for the standard LRSM by cross validation (so that it serves as a stronger baseline). In addition, we set the prior $W_0$ in our joint model as the edge weight matrix of the original regularization graph in \cite{tuck2019distributed}.
Further more, we choose $\lambda_1=100, \lambda_2=12, \eta=0.1, \mu=20.0$ by 5-fold cross validation.

\paragraph{Results.} Table \ref{tab_mes} shows the test prediction error, ANLL, AUC and F1 score of the two models. We see that the joint-stratified model outperforms the standard one consistently. Figure \ref{fig:mes_diff_wt} demonstrates the heatmap of the weight difference $W-W_0$. Note that the edge weight of $W_0$ is either $10$ or $500$, and hence the magnitude of the difference is relatively small. This further demonstrates that the performance of Laplacian regularization is sensitive to the (high-dimensional) weights, making  it hard to tune by hand. Our joint learning method can therefore serve as a handy tuning procedure in this case.

\begin{table}[ht]
\centering{\small
\scalebox{0.95}{
\begin{tabular}{l|c|c|c|c}
\hline
& Error ($\%$) & ANLL & AUC & F1-Score \\
\hline
Common & $33.3$ & $0.728$ & $0.529$ & $0.593$\\
\hline
Separate & $39.4$ & $0.695$ & $0.402$ & $0.439$\\
\hline
LRSM & $24.2$ & $0.532$ & $0.582$ & $0.592$\\
\hline
Joint Laplacian stratified & $\pmb{21.2}$ & $\pmb{0.479}$ & $\pmb{0.583}$ & $\pmb{0.618}$\\
\hline
\end{tabular}
}
\caption{Mesothelioma results.}
\label{tab_mes}
}
\end{table}

\begin{figure}[h]
    \centering
    \includegraphics[scale=0.65]{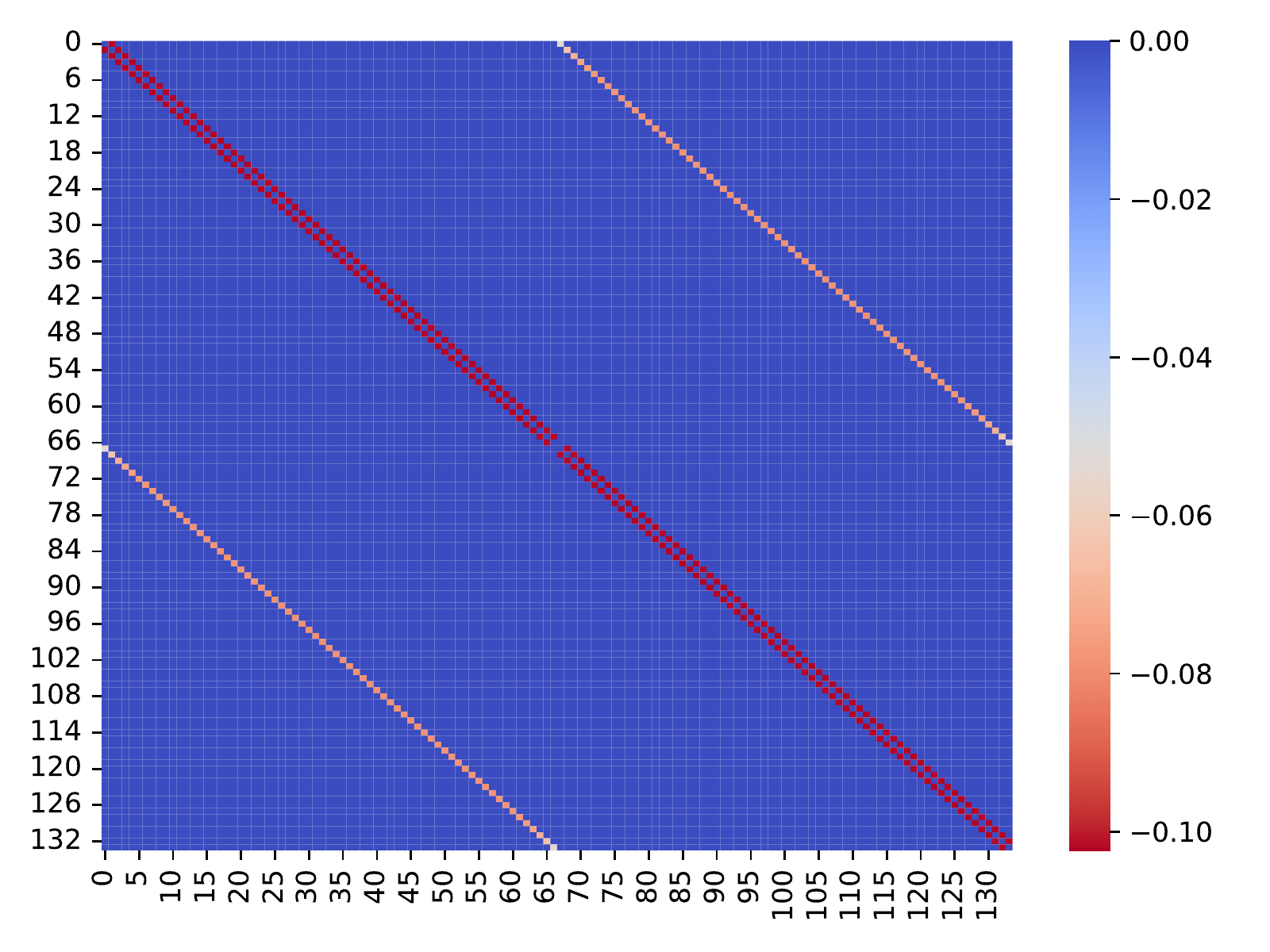}
    \caption{Heatmap of weight difference $W-W_0$.}
    \label{fig:mes_diff_wt}
\end{figure}


\subsubsection{Senate elections}

We model the probability that a United States Senate election in a particular state and election year is won by the Democratic party. For each state and election year, we have a single Bernoulli parameter that can be interpreted as the probability that state will elect a candidate from the Democratic party. Our loss function is the negative log-likelihood. There are 639 training records and 68 test records.

The setting of the model and baseline algorithms are the same as \cite{tuck2019distributed}. In addition, we set the prior $W_0$ in our joint model as the edge weight matrix of the original regularization graph in \cite{tuck2019distributed}. And we set $\lambda_1=0.2, \lambda_2=2.0, \eta=1\times 10^{-3}, \mu=0.1$ by performing 5-fold cross validation.

\paragraph{Results.}

\begin{table}[ht]
\centering{\small
\scalebox{0.95}{
\begin{tabular}{l|c|c|c|c}
\hline
& Accuracy ($\%$) & ANLL & AUC & F1-Score\\
\hline
Common & $35.3$ & $0.701$ & $0.500$ & $0.522$\\
\hline
Separate & $35.3$ & $0.998$ & $0.500$ & $0.522$\\
\hline
LRSM & $72.1$ & $0.608$ & $0.756$ & $0.689$\\
\hline
Joint Laplacian stratified & $\pmb{80.9}$ & $\pmb{0.546}$ & $\pmb{0.814}$ & $\pmb{0.755}$ \\
\hline
\end{tabular}
}
\caption{Election prediction results.}
\label{tab_election}
}
\end{table}


\begin{figure}[h]
    \centering
    \includegraphics[scale=0.45]{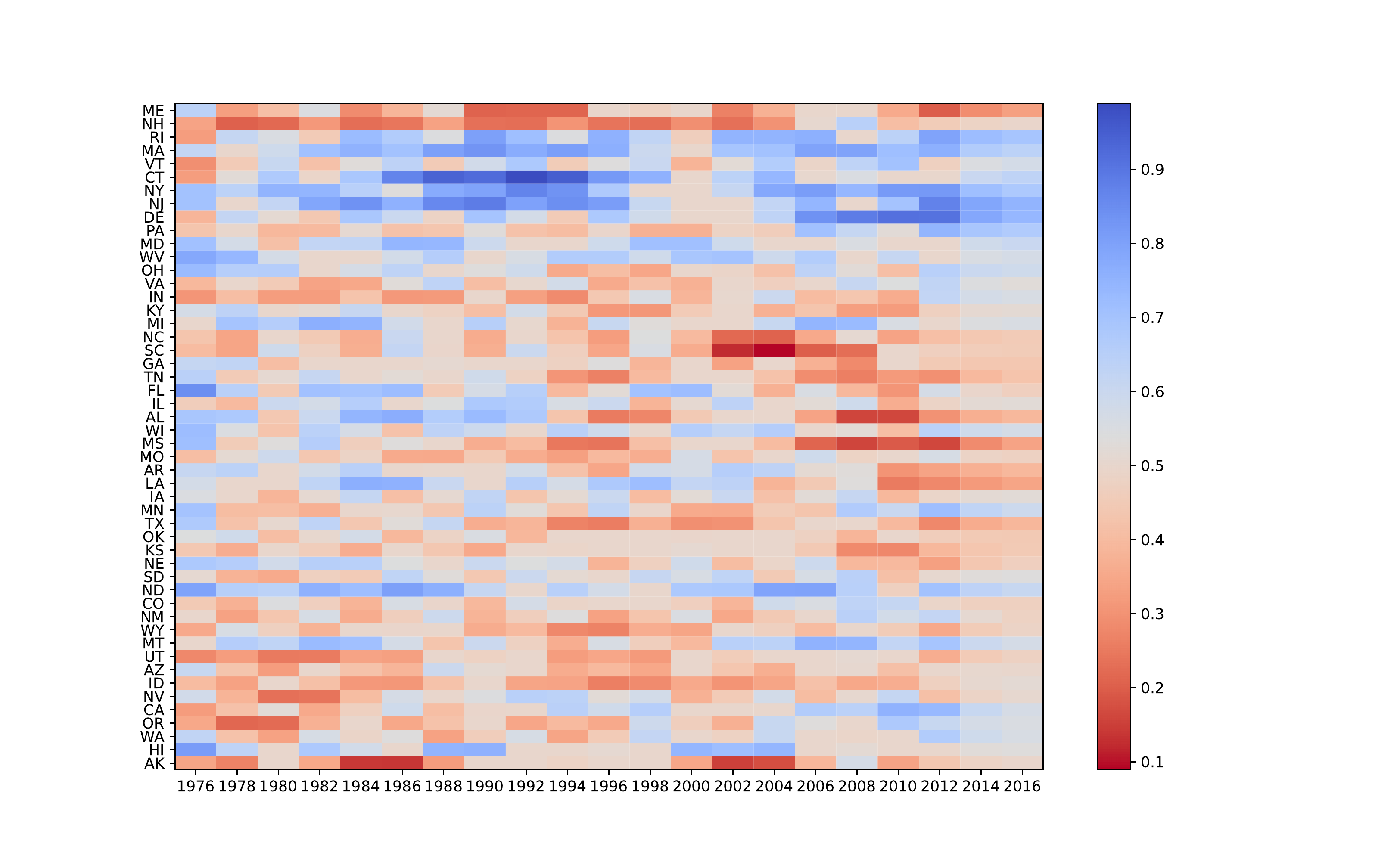}
    \caption{The Bernoulli parameters across election year and state in LRSM.}
    \label{fig:election_std}
\end{figure}

\begin{figure}[h]
    \centering
    \includegraphics[scale=0.45]{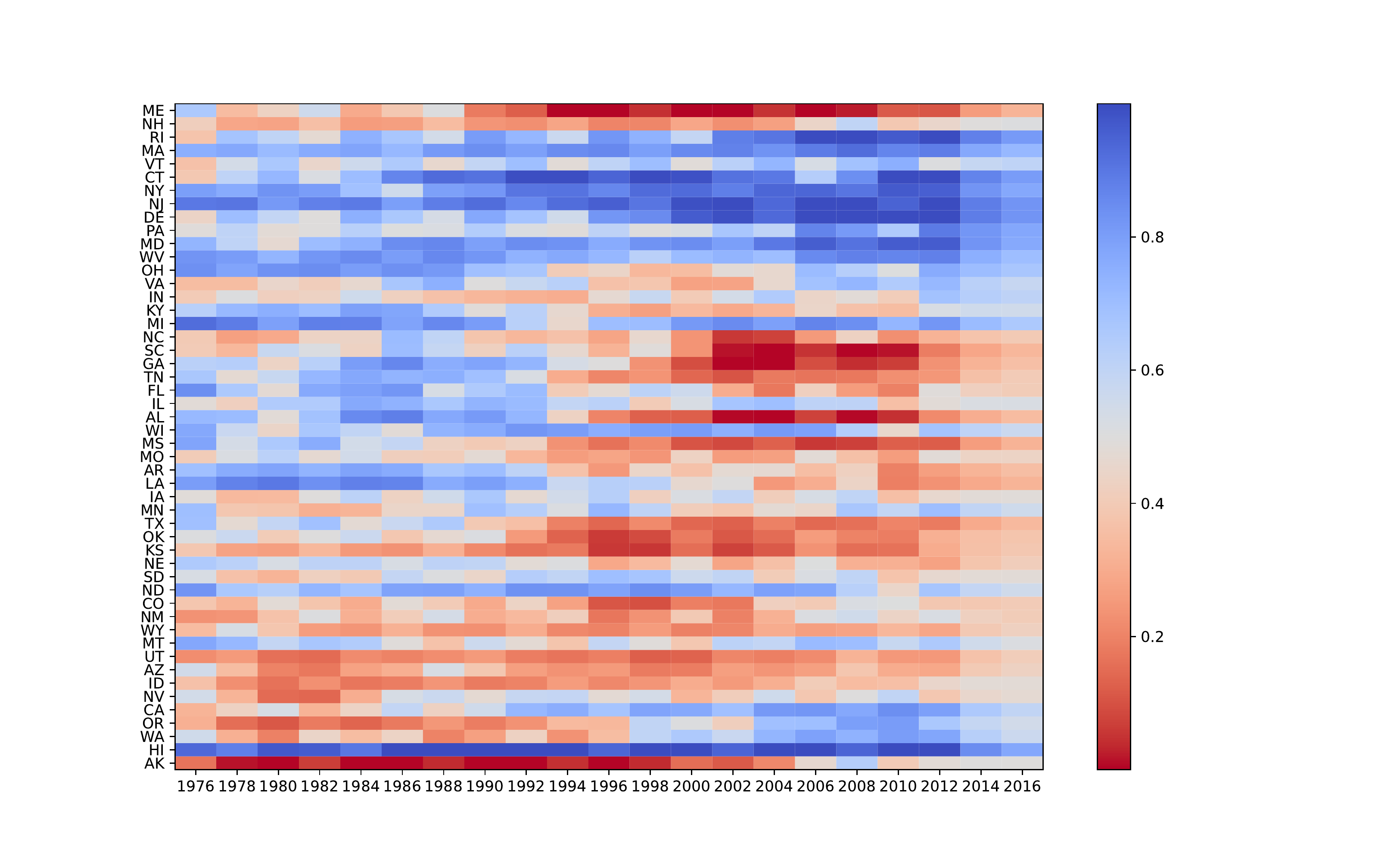}
    \caption{The Bernoulli parameters across election year and state in Joint Laplacian stratified model.}
    \label{fig:election_joint}
\end{figure}

Table \ref{tab_election} shows the test accuracy, ANLL, AUC and F1 score of the two models. We see that the Joint Laplacian stratified model outperforms the standard one significantly. Figure \ref{fig:election_std} and Figure \ref{fig:election_joint} demonstrate the bernoulli parameter of two models. High model parameters correspond
to blue (Democrat) and low model parameters correspond to red (Republican and other
parties). Note that in joint stratified model, the color is generally darker, indicating the parameter is closer to 0 or 1.

\subsection{Stratified model without prior}
In this subsection, we set $W_0$ as the zero matrix in the Joint Laplacian stratified model, while for the standard LRSM, we design the graph by hand following a similar pattern as is done in \cite{tuck2019distributed} and tune the edge weights by cross validation. With $W_0$ being zero, we do not have any prior information and the weight matrix $W$ is learned automatically in the training process.  We also consider the joint-learning baselines mentioned before, namely Log-Diagonal and Tr-Constraint. The goal is to see if our method can also compete with the standard one or even perform better than it.

Since the theoretical analysis of our MAPG algorithm can also be easily generalized to the two settings above, we apply it to these two baselines. We also consider alternating method to solve the optimization problem as in the literature. The results of Log-Diagonal and Tr-Constraint shown are the best of both optimization algorithms.

\subsubsection{Wine quality classfication}
We are aimed to predict the red wine quality given its physicochemical and sensory features. 

\paragraph{Dataset.} We get dataset describing the quality of red variants of the Portuguese "Vinho Verde" wine. There are 11 features, including fixed acidity, volatile acidity, citric acid, residual sugar, chlorides, free sulfur dioxide, total sulfur dioxide, density, pH, sulphates, alcohol. The output is a binary variable representing the quality of wine. 

\paragraph{Data records.} There are 1599 instances in the dataset. The stratification feature is density and sulphates. Since alcohol is highly correlated with the quality, we removed this feature. Therefore, we have feature vector $x\in \reals^9$ with intercept. We take $20\%$ of the dataset as test set and the rest serves as training set. We randomly split the training data in five folds and standardized the features so that they have zero mean and unit variance.

\paragraph{Data model.} Our model is logistic regression and local $\ell_1$-regularization with weight $\gamma_{local}$.

\paragraph{Regularization graph for the standard LRSM.}
We take the Cartesian product of two regularization graphs:
\begin{itemize}
    \item Density. We bin the density into 10 equally sized bins. The regularization graph is a path graph between density bins, with edge weight $\gamma_{den}$.
    \item Sulphates. We bin the sulphates feature into 10 equally sized bins. The regularization graph is a path graph between bins, with edge weight $\gamma_{sul}$.
\end{itemize}
Therefore we have $K=10\times 10=100$ stratified models. Note that the regularization graph structure above is only used for the standard Laplacian stratified models, while for our method we make no assumption on the graph except for that it consists of these $100$ strata nodes.

\paragraph{Results.} In LRSM, we set $\gamma_{den}=20,\gamma_{sul}=10$ by 5-fold cross-validation. In the Joint Laplacian stratified models, we set $\lambda_1=5.0, \lambda_2=2\times 10^{-3}, \eta=0.1, \mu=0.2$ by 5-fold cross validation. The local regularization weight is $\gamma_{local}=0.01$. 

\begin{table}[H]
\centering{\small
\scalebox{0.95}{
\begin{tabular}{l|c|c|c|c}
\hline
& Error ($\%$) & ANLL & AUC & F1-Score \\
\hline
Common & $31.2$ & $0.585$ & $0.761$ & $0.679$\\
\hline
Separate & $28.1$ & $0.827$ & $0.689$ & $0.715$\\
\hline
LRSM & $27.8$ & $\pmb{0.537}$ & $\pmb{0.802}$ & $0.717$\\
\hline
Log-Diagonal stratified & $28.4$ & $0.621$ & $0.625$ & $0.712$\\
\hline
Tr-Constraint stratified & $27.5$ & $0.781$ & $0.692$ & $0.721$\\
\hline
Joint Laplacian stratified & $\pmb{26.6}$ & $0.549$ & $0.789$ & $\pmb{0.729}$\\
\hline
\end{tabular}
}
\caption{Wine results.}
\label{tab_wine}
}
\end{table}

\begin{figure}[h]
    \centering
    \includegraphics[scale=0.6]{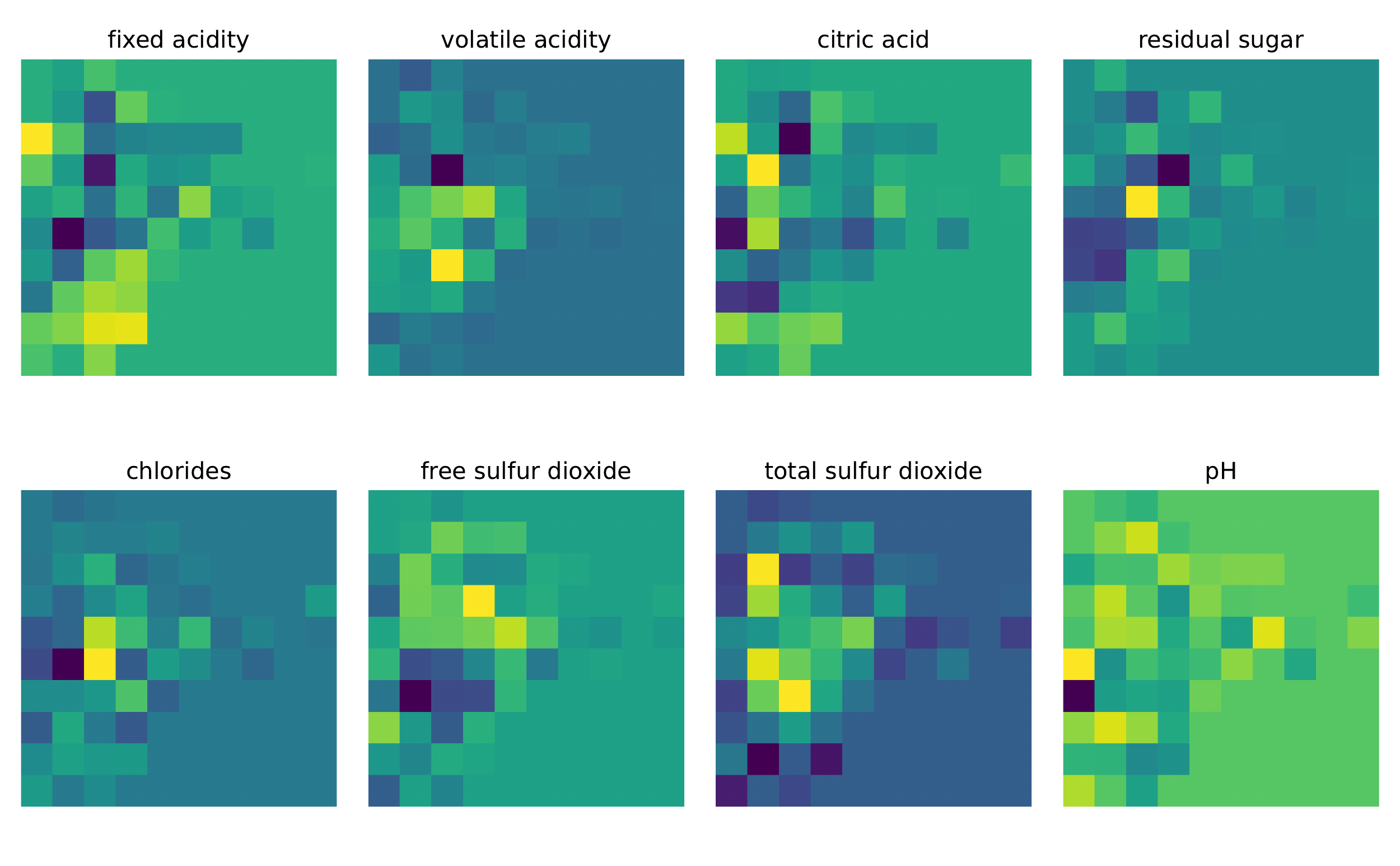}
    \caption{Wine quality coefficients in Joint Laplacian stratified model.}
    \label{fig:wine_coef}
\end{figure}

\begin{figure}[h]
    \centering
    \includegraphics[scale=0.4]{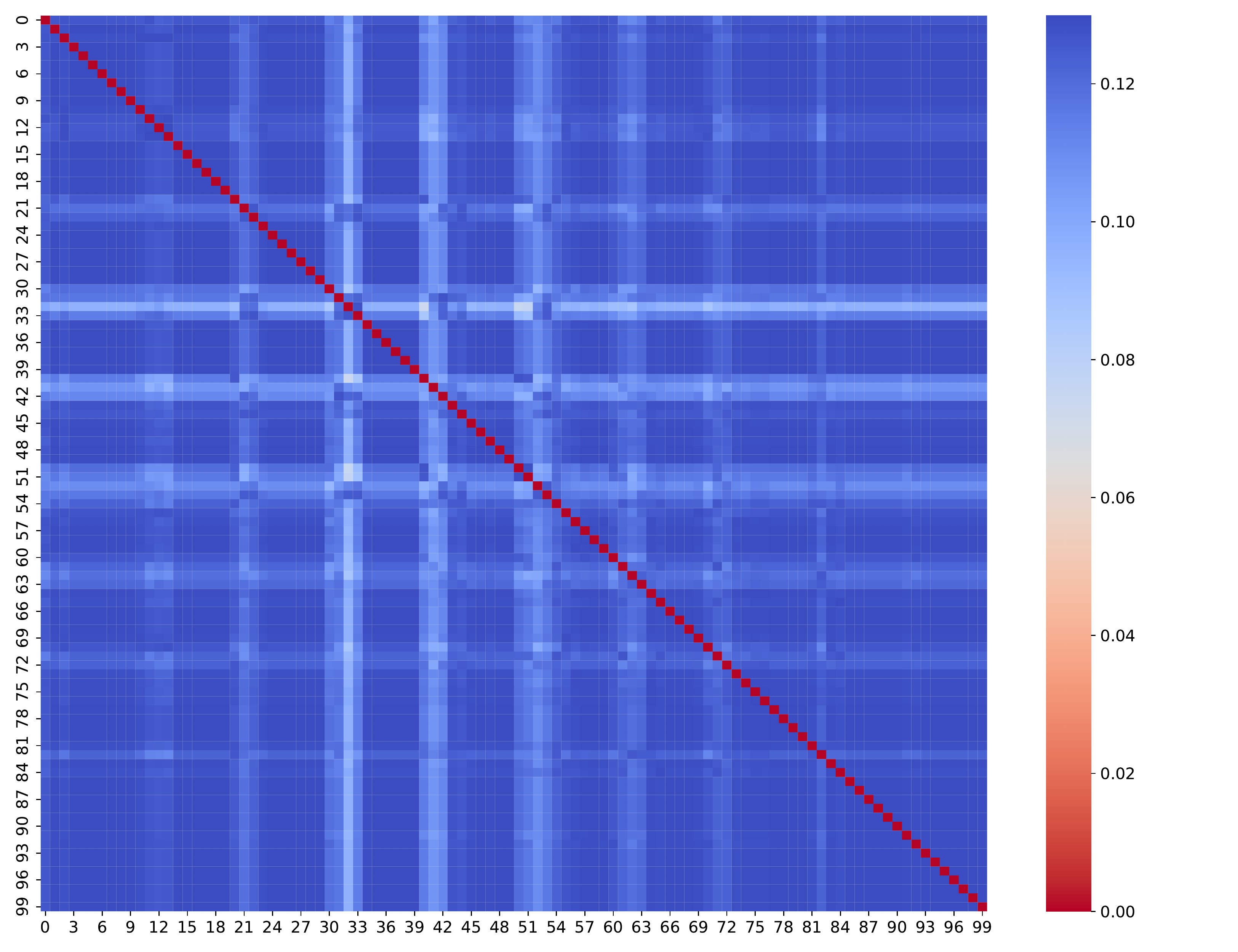}
    \caption{Wine quality edge weight matrix learnt by Joint Laplacian stratified models.}
    \label{fig:wine_wt_coef}
\end{figure}


Table \ref{tab_wine} shows four evaluation metrics on the test set, indicating that the Joint Laplacian stratified model does just as well as the standard one. Our joint optimization method accurately tunes the edge weight, even without the availability of prior knowledge. Figure \ref{fig:wine_coef} shows the 8 coefficients for each bins in Joint Laplacian stratified model. Figure \ref{fig:wine_wt_coef} visualizes the edge weight learnt by our method.

\subsubsection{Concrete strength prediction}
We model the conditional quantile of concrete compressive strength given their age and ingredients like Fly ash, cement.

\paragraph{Dataset.} We obtained data describing the concrete compressive strength. The concrete is comprised of 7 components, including cement, Blast Furnace Slag, Fly Ash, Water, Superplasticizer, Coarse Aggregate, Fine Aggregate. Besides, the dataset contains information about the age of concrete. 

\paragraph{Data records.} There are 1030 instances in the dataset. The stratification feature is age and fly ash component. Therefore, we have feature vector $x\in \reals^{7}$ with intercept. We take $25\%$ of the dataset as test set and the rest serves as training set. We randomly split the training data in five folds and standardized the features so that they have zero mean and unit variance.

\paragraph{Data model.} Our model is quantile regression in statistics \cite{koenker1978regression}. Instead of conditional mean in OLS, quantile regression is targeted at the $\tau$-th conditional quantile of $y$ given $\bm{X}$, which is defined as $Q_{Y\vert \bm{X}}(\tau):=\inf\{y: \mathbf{F}_{Y\vert \bm{X}}(y)\geq \tau\}$ and $\mathbf{F}_{Y\vert \bm{X}}$ is the conditional \textbf{CDF}.

The local loss function is given by pinball loss:
\begin{equation*}
    l_k(\theta)=\sum_{i=1}^{n_k}\rho_{\tau}(y_i-\bm{x}_i^T \bm{\theta}),
\end{equation*}
where $\rho_{\tau}(z)=(1-\tau)\max\{-z, 0\}+\tau \max\{z, 0\}$. For the joint stratified model, because the loss function is non-smooth at zero, we use Huber loss to compute the gradient instead. For LRSM, the proximal operator of $l_k$ is estimated by calling CVXPY. And we also use the sum of squares local regularization function with regularization weight $\gamma_{local}$.

\paragraph{Regularization graph for the standard LRSM.} We take the Cartesian product of two regularization graphs:
\begin{itemize}
    \item Age. We bin the age into 10 equally sized bins. The regularization graph is a path graph between age bins, with edge weight $\gamma_{age}$.
    \item Fly ash. We bin the Fly ash feature into 10 equally sized bins. The regularization graph is a path graph between bins, with edge weight $\gamma_{ash}$.
\end{itemize}
Therefore we have $K=10\times 10=100$ stratified models. Again, note that for our Joint Laplacian stratified model, we make no assumption on the graph except for the same $100$ strata nodes.

\paragraph{Results.} Our goal is to find the $\tau=0.9$-th quantile. In LRSM, we set $\gamma_{age}=\gamma_{ash}=0.5$. In the Joint Laplacian stratified models, we set $\lambda_1=5.0, \lambda_2=10.0, \eta=1\times {10}^{-3}, \mu=2\times {10}^{-3}$. The local regularization weight is $\gamma_{local}=0.01$.

\begin{table}[H]
\centering{\small
\scalebox{0.95}{
\begin{tabular}{l|c|c}
\hline
& Train loss & Test loss\\
\hline
Common & $2.06$ & $2.27$ \\
\hline
Separate & $\pmb{1.23}$ & $2.52$ \\
\hline
LRSM & $1.56$ & $1.74$ \\
\hline
Log-Diagonal stratified & $1.59$ & $1.99$ \\
\hline
Tr-Constraint stratified & $1.37$ & $1.85$ \\
\hline
Joint Laplacian stratified & $1.49$ & $\pmb{1.58}$\\
\hline
\end{tabular}
}
\caption{Concrete results.}
\label{tab_concrete}
}
\end{table}



Table \ref{tab_concrete} presents the pinball loss results for both training set and test set. The Joint Laplacian stratified model outperforms the standard one significantly. Notably, our joint optimization method efficiently tunes the edge weight without the availability of prior knowledge. Figure \ref{fig:concrete_ce_wt} demonstrates the bar plot of cement and water parameter in Joint Laplacian stratified models. Note that the parameters across the graph are not very smooth. The regularization merely on node degree is insufficient to control the graph connection structure. Therefore Log-Diagonal and Tr-Constraint do worse than our method. Figure \ref{fig:concrete_wt} is the visualization of edge weight matrix learnt by Joint Laplacian stratified models, which is very difficult to get by tuning by hand.


\begin{figure}[h]
    \centering
    \includegraphics[scale=0.55]{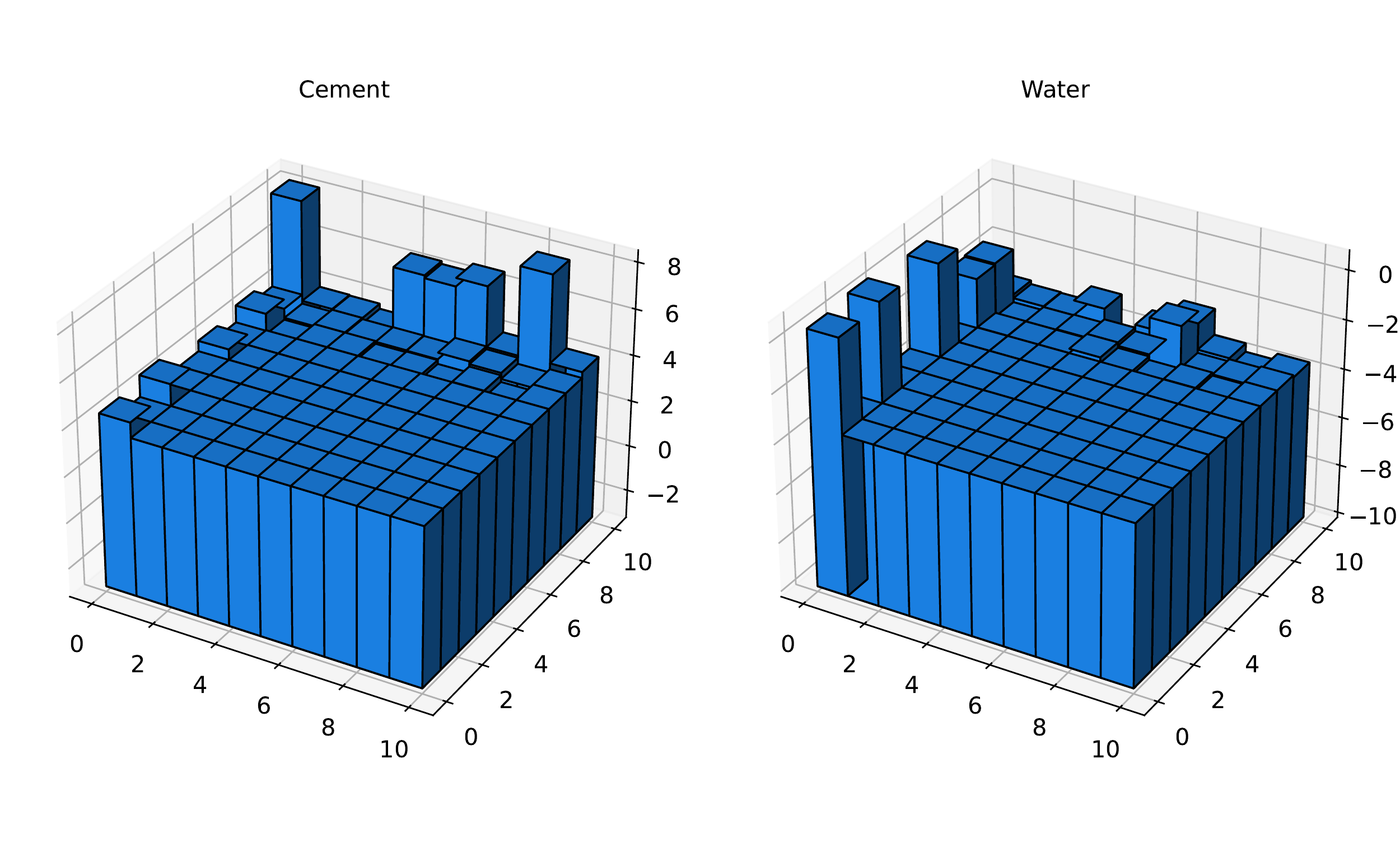}
    \caption{Cement and Water parameter across sub-models.}
    \label{fig:concrete_ce_wt}
\end{figure}

\begin{figure}[H]
    \centering
    \includegraphics[scale=0.4]{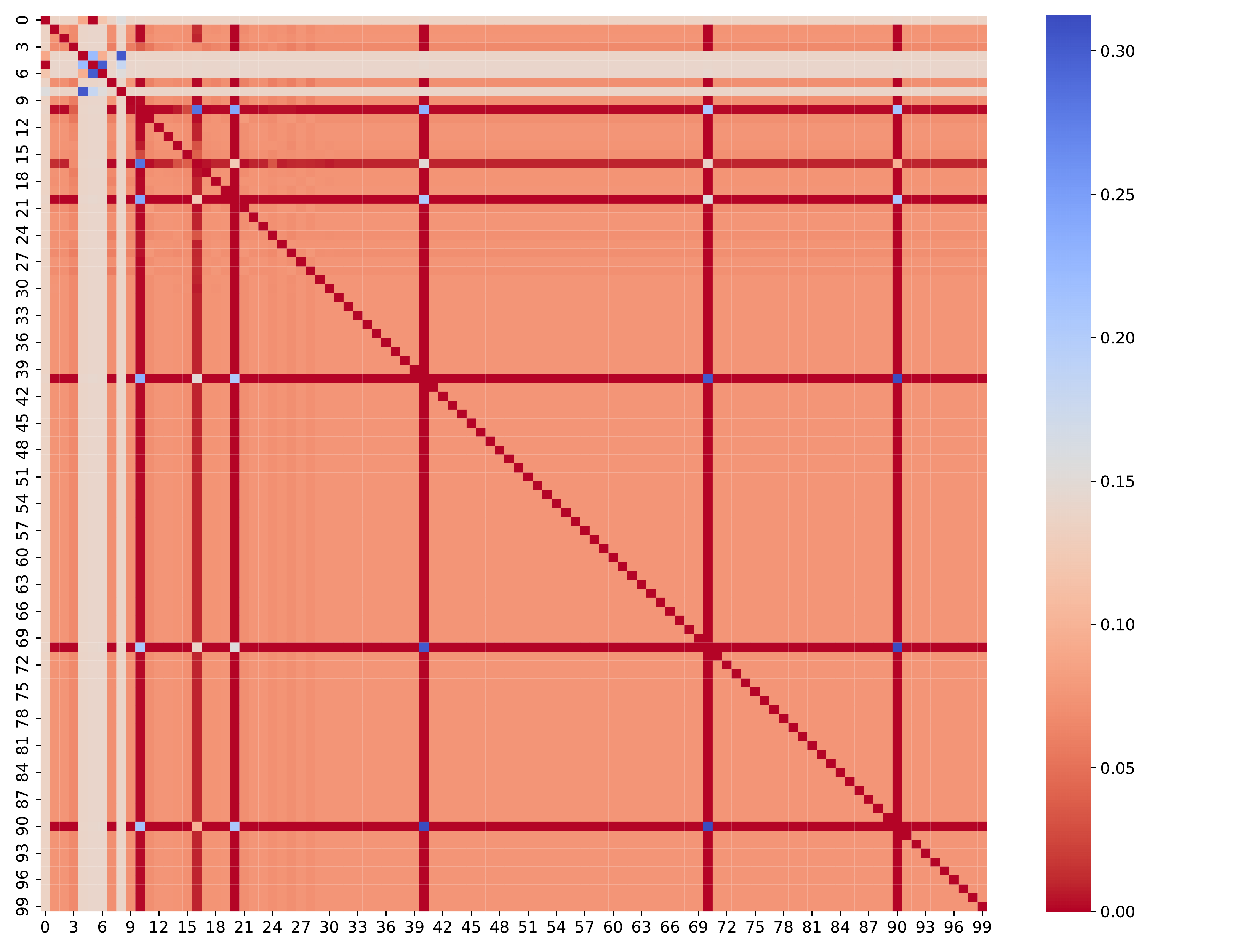}
    \caption{Concrete edge weight learnt by Joint Laplacian stratified model.}
    \label{fig:concrete_wt}
\end{figure}

\section*{Acknowledgments}
The authors would like to express their deep gratitude to Jonathan Tuck and Anran Hu for their valuable contributions to this work. Their insightful discussions, feedback, and support were crucial to the success of this project.

\bibliographystyle{alpha}
\bibliography{joint_lap_strat.bib}

\newcommand{\etalchar}[1]{$^{#1}$}
\begin{thebibliography}{KYdMCP19}

\bibitem[AAB{\etalchar{+}}19]{agrawal2019differentiable}
Akshay Agrawal, Brandon Amos, Shane Barratt, Stephen Boyd, Steven Diamond, and
  Zico Kolter.
\newblock Differentiable convex optimization layers.
\newblock {\em arXiv preprint arXiv:1910.12430}, 2019.

\bibitem[AZ06]{ando2006learning}
Rie Ando and Tong Zhang.
\newblock Learning on graph with laplacian regularization.
\newblock {\em Advances in neural information processing systems}, 19, 2006.

\bibitem[BB20]{barratt2020least}
Shane~T Barratt and Stephen~P Boyd.
\newblock Least squares auto-tuning.
\newblock {\em Engineering Optimization}, pages 1--22, 2020.

\bibitem[BST14]{bolte2014proximal}
J{\'e}r{\^o}me Bolte, Shoham Sabach, and Marc Teboulle.
\newblock Proximal alternating linearized minimization for nonconvex and
  nonsmooth problems.
\newblock {\em Mathematical Programming}, 146(1):459--494, 2014.

\bibitem[CK78]{chaiken1978matrix}
Seth Chaiken and Daniel~J Kleitman.
\newblock Matrix tree theorems.
\newblock {\em Journal of combinatorial theory, Series A}, 24(3):377--381,
  1978.

\bibitem[CPVBR20]{Cabannes2020OvercomingTC}
Vivien~A. Cabannes, Loucas Pillaud-Vivien, Francis~R. Bach, and Alessandro
  Rudi.
\newblock Overcoming the curse of dimensionality with laplacian regularization
  in semi-supervised learning.
\newblock In {\em Neural Information Processing Systems}, 2020.

\bibitem[CS06a]{chebotarev2006matrix_social}
Pavel Chebotarev and Elena Shamis.
\newblock The matrix-forest theorem and measuring relations in small social
  groups.
\newblock {\em arXiv preprint math/0602070}, 2006.

\bibitem[CS06b]{chebotarev2006matrix}
Pavel Chebotarev and Elena Shamis.
\newblock Matrix-forest theorems.
\newblock {\em arXiv preprint math/0602575}, 2006.

\bibitem[DTFV14]{Dong2014LearningLM}
Xiaowen Dong, Dorina Thanou, Pascal Frossard, and Pierre Vandergheynst.
\newblock Learning laplacian matrix in smooth graph signal representations.
\newblock {\em IEEE Transactions on Signal Processing}, 64:6160--6173, 2014.

\bibitem[DVT{\etalchar{+}}21]{Dinh2021ANL}
Canh~T. Dinh, Thanh~Tung Vu, Nguyen~H. Tran, Minh~N. Dao, and Hongyu Zhang.
\newblock A new look and convergence rate of federated multitask learning with
  laplacian regularization.
\newblock {\em IEEE Transactions on Neural Networks and Learning Systems},
  2021.

\bibitem[DW22]{Duan2022AdaptiveAR}
Yaqi Duan and Kaizheng Wang.
\newblock Adaptive and robust multi-task learning.
\newblock {\em ArXiv}, abs/2202.05250, 2022.

\bibitem[EP04]{Evgeniou2004RegularizedML}
Theodoros Evgeniou and Massimiliano Pontil.
\newblock Regularized multi--task learning.
\newblock {\em Proceedings of the tenth ACM SIGKDD international conference on
  Knowledge discovery and data mining}, 2004.

\bibitem[GVZB16]{gonccalves2016multi}
Andr{\'e}~R Gon{\c{c}}alves, Fernando~J Von~Zuben, and Arindam Banerjee.
\newblock Multi-task sparse structure learning with gaussian copula models.
\newblock {\em The Journal of Machine Learning Research}, 17(1):1205--1234,
  2016.

\bibitem[HTK22]{He2022StockPW}
Jiayu He, Nguyen~H. Tran, and Matloob Khushi.
\newblock Stock predictor with graph laplacian-based multi-task learning.
\newblock In {\em International Conference on Conceptual Structures}, 2022.

\bibitem[JVB08]{jacob2008clustered}
Laurent Jacob, Jean-philippe Vert, and Francis Bach.
\newblock Clustered multi-task learning: A convex formulation.
\newblock {\em Advances in neural information processing systems}, 21, 2008.

\bibitem[KA22]{Karaaslanli2022SimultaneousGS}
Abdullah Karaaslanli and Selin Aviyente.
\newblock Simultaneous graph signal clustering and graph learning.
\newblock In {\em International Conference on Machine Learning}, 2022.

\bibitem[Kal16]{Kalofolias2016HowTL}
Vassilis Kalofolias.
\newblock How to learn a graph from smooth signals.
\newblock In {\em International Conference on Artificial Intelligence and
  Statistics}, 2016.

\bibitem[KBJ78]{koenker1978regression}
Roger Koenker and Gilbert Bassett~Jr.
\newblock Regression quantiles.
\newblock {\em Econometrica: journal of the Econometric Society}, pages 33--50,
  1978.

\bibitem[KLZX21]{Kang2021StructuredGL}
Zhao Kang, Zhiping Lin, Xiaofeng Zhu, and Wenbo Xu.
\newblock Structured graph learning for scalable subspace clustering: From
  single view to multiview.
\newblock {\em IEEE Transactions on Cybernetics}, 52:8976--8986, 2021.

\bibitem[KP17]{Kalofolias2017LargeSG}
Vassilis Kalofolias and Nathanael Perraudin.
\newblock Large scale graph learning from smooth signals.
\newblock {\em ArXiv}, abs/1710.05654, 2017.

\bibitem[KYdMCP19]{Kumar2019AUF}
Sandeep Kumar, Jiaxi Ying, Jos{\'e}~Vin{\'i}cius de~Miranda~Cardoso, and
  Daniel~P{\'e}rez Palomar.
\newblock A unified framework for structured graph learning via spectral
  constraints.
\newblock {\em J. Mach. Learn. Res.}, 21:22:1--22:60, 2019.

\bibitem[LL15]{NIPS2015_f7664060}
Huan Li and Zhouchen Lin.
\newblock Accelerated proximal gradient methods for nonconvex programming.
\newblock In C.~Cortes, N.~Lawrence, D.~Lee, M.~Sugiyama, and R.~Garnett,
  editors, {\em Advances in Neural Information Processing Systems}, volume~28.
  Curran Associates, Inc., 2015.

\bibitem[LWWZ22]{Lam2022AdaptiveDF}
Henry Lam, Kaizheng Wang, Yuhang Wu, and Yichen Zhang.
\newblock Adaptive data fusion for multi-task non-smooth optimization.
\newblock {\em ArXiv}, abs/2210.12334, 2022.

\bibitem[LZZ{\etalchar{+}}16]{Li2016JointMC}
Xi~Li, Xueyi Zhao, Zhongfei Zhang, Fei Wu, Yueting Zhuang, Jingdong Wang, and
  Xuelong Li.
\newblock Joint multilabel classification with community-aware label graph
  learning.
\newblock {\em IEEE Transactions on Image Processing}, 25:484--493, 2016.

\bibitem[PB{\etalchar{+}}14]{parikh2014proximal}
Neal Parikh, Stephen Boyd, et~al.
\newblock Proximal algorithms.
\newblock {\em Foundations and trends{\textregistered} in Optimization},
  1(3):127--239, 2014.

\bibitem[PCDS20]{Pu2020KernelBasedGL}
Xingyue Pu, Siu~Lun Chau, Xiaowen Dong, and D.~Sejdinovic.
\newblock Kernel-based graph learning from smooth signals: A functional
  viewpoint.
\newblock {\em IEEE Transactions on Signal and Information Processing over
  Networks}, 7:192--207, 2020.

\bibitem[Pen17]{mat_tree_lec}
Richard Peng.
\newblock {CS}7540 {L}ecture \#2, {M}atrix tree theorem.
\newblock
  \url{https://www.cc.gatech.edu/~rpeng/CS7540_S17/Jan12MatrixTree.pdf}, 2017.

\bibitem[RH05]{Rue2005GaussianMR}
H{\aa}vard Rue and Leonhard Held.
\newblock Gaussian markov random fields: Theory and applications.
\newblock 2005.

\bibitem[SCST17]{Smith2017FederatedML}
Virginia Smith, Chao-Kai Chiang, Maziar Sanjabi, and Ameet~S. Talwalkar.
\newblock Federated multi-task learning.
\newblock In {\em NIPS}, 2017.

\bibitem[She08]{sheldon2008graphical}
Daniel Sheldon.
\newblock Graphical multi-task learning.
\newblock 2008.

\bibitem[ST17]{Slepev2017AnalysisO}
Dejan Slep{\v{c}}ev and Matthew Thorpe.
\newblock Analysis of \$p\$-laplacian regularization in semi-supervised
  learning.
\newblock {\em ArXiv}, abs/1707.06213, 2017.

\bibitem[TB20a]{tuck2020eigen}
Jonathan Tuck and Stephen Boyd.
\newblock Eigen-stratified models.
\newblock {\em arXiv preprint arXiv:2001.10389}, 2020.

\bibitem[TB20b]{tuck2020fitting}
Jonathan Tuck and Stephen Boyd.
\newblock Fitting {L}aplacian regularized stratified {G}aussian models.
\newblock {\em arXiv preprint arXiv:2005.01752}, 2020.

\bibitem[TBB19]{tuck2019distributed}
Jonathan Tuck, Shane Barratt, and Stephen Boyd.
\newblock A distributed method for fitting {L}aplacian regularized stratified
  models.
\newblock {\em arXiv preprint arXiv:1904.12017}, 2019.

\bibitem[TBB21]{tuck2021portfolio}
Jonathan Tuck, Shane Barratt, and Stephen Boyd.
\newblock Portfolio construction using stratified models.
\newblock {\em arXiv preprint arXiv:2101.04113}, 2021.

\bibitem[VVBD20]{VargasVieyra2020JointLO}
Mariana Vargas-Vieyra, Aur{\'e}lien Bellet, and Pascal Denis.
\newblock Joint learning of the graph and the data representation for
  graph-based semi-supervised learning.
\newblock {\em Proceedings of the Graph-based Methods for Natural Language
  Processing (TextGraphs)}, 2020.

\bibitem[ZCY11]{zhou2011clustered}
Jiayu Zhou, Jianhui Chen, and Jieping Ye.
\newblock Clustered multi-task learning via alternating structure optimization.
\newblock {\em Advances in neural information processing systems}, 24, 2011.

\bibitem[ZDGA20]{ziko2020laplacian}
Imtiaz Ziko, Jose Dolz, Eric Granger, and Ismail~Ben Ayed.
\newblock Laplacian regularized few-shot learning.
\newblock In {\em International Conference on Machine Learning}, pages
  11660--11670. PMLR, 2020.

\bibitem[ZLWN21]{Zhao2021JointAG}
Haifeng Zhao, Qi~Li, Z.~Wang, and Feiping Nie.
\newblock Joint adaptive graph learning and discriminative analysis for
  unsupervised feature selection.
\newblock {\em Cognitive Computation}, 14:1211--1221, 2021.

\bibitem[ZY10]{Zhang2010ACF}
Yu~Zhang and D.~Y. Yeung.
\newblock A convex formulation for learning task relationships in multi-task
  learning.
\newblock In {\em Conference on Uncertainty in Artificial Intelligence}, 2010.

\end{thebibliography}

\newpage
\appendix

\section{Proof of Theorem \ref{thm:sensitivity}}
\label{appendix_motivation}
\begin{proof}
    Let $S = \sum_{j\neq k_0}W_{jk_0}$. Since strata $k_0$ have no data sample and zero local regularization, then \eqref{eq:lap_strat} indicates that
    \begin{equation}
        \theta_{k_0} = \frac{\sum_{j\neq k_0} W_{jk_0} \theta_j}{S},\
        \widetilde{\theta}_{k_0} = \frac{\sum_{j\neq k_0, i} W_{jk_0} \widetilde{\theta_j} + \widetilde{W}_{ik_0} \widetilde{\theta_i}}{S+\epsilon}.
    \end{equation}
    Note that both $\Theta, \widetilde{\Theta}$ are $\delta$-robust by assumption. Subtract the two equations above and rearrange it,
    \begin{equation*}
        \begin{aligned}
            ||\frac{W_{ik_0}}{S}\theta_i - \frac{W_{ik_0}+\epsilon}{S+\epsilon}\widetilde{\theta}_i|| 
            &\leq ||\theta_{k_0}-\widetilde{\theta}_{k_0}|| + \sum_{j\neq k_0,i} ||\frac{W_{jk_0}\theta_j}{S} - \frac{W_{jk_0}\widetilde{\theta}_j}{S+\epsilon}|| \\
            &\leq 2\delta'||\theta_{k_0}^*|| + \sum_{j\neq k_0,i} \left[ \frac{W_{jk_0}}{S} ||\theta_j-\theta_j^*|| + \frac{W_{jk_0}}{S+\epsilon} ||\widetilde{\theta}_j-\theta_j^*|| + \frac{\epsilon W_{jk_0}}{S(S+\epsilon)} ||\theta_j^*||\right] \\
            &\leq 2\delta'||\theta_{k_0}^*|| + \sum_{j\neq k_0,i} \left( \frac{\delta W_{jk_0}}{S} + \frac{\delta W_{jk_0}}{S+\epsilon} + \frac{\epsilon W_{jk_0}}{S(S+\epsilon)} \right) ||\theta_j^*|| \\
            &\leq \max_{j\neq i} ||\theta_j^*|| \left( 2\delta' + 2\delta + \frac{\epsilon}{S+\epsilon}(1-\frac{W_{ik_0}}{S}) \right).
        \end{aligned}
    \end{equation*}
    Additionally, we have
    \begin{equation*}
        \begin{aligned}
            ||\frac{W_{ik_0}}{S}\theta_i - \frac{W_{ik_0}+\epsilon}{S+\epsilon}\widetilde{\theta}_i||
            &\geq (-\frac{W_{ik_0}}{S} + \frac{W_{ik_0}+\epsilon}{S+\epsilon}) ||\theta_i^*|| - \frac{W_{ik_0}}{S}||\theta_i - \theta_i^*|| - \frac{W_{ik_0}+\epsilon}{S+\epsilon} ||\widetilde{\theta}_i - \theta_i^*|| \\
            &\geq \left( \frac{\epsilon}{S+\epsilon}(1-\frac{W_{ik_0}}{S}) -\frac{W_{ik_0}}{S}\delta-\frac{W_{ik_0}+\epsilon}{S+\epsilon}\delta \right)||\theta_i^*|| \\
            &\geq \left( \frac{\epsilon}{S+\epsilon}(1-\frac{W_{ik_0}}{S}-\delta) -2\frac{W_{ik_0}}{S}\delta \right)||\theta_i^*||.
        \end{aligned}
    \end{equation*}
    Combine the above two inequalities and let $A_i:= \frac{||\theta_i^*||}{\max_{j\neq i} ||\theta_j^*||}$, we obtain (provided that the denominator is positive),
    \begin{equation*}
        \frac{\epsilon}{W_{ik_0}} \leq \frac{2(\delta+\delta'+\frac{\delta W_{ik_0}}{S}A_i)}{\left(1-\frac{(1+2\delta)W_{ik_0}}{S}-\delta\right)A_i-1+\frac{W_{ik_0}}{S}-2(\delta+\delta')} \frac{S}{W_{ik_0}}.
    \end{equation*}
\end{proof}

\section{Proof of Theorem \ref{thm_joint_converge}}
\label{appendix_proof}

We first state an important lemma in our proof of Theorem \ref{thm_joint_converge}.
\begin{lemma}[\cite{bolte2014proximal}]
\label{lemma_kl}
Let  $\Omega$  be a compact set and let $F: \reals^{n} \rightarrow(-\infty,+\infty]$  be a proper and lower semicontinuous function. Assume that  $F$  is constant on  $\Omega$ and satisfies the KL property at each point of $\Omega$. Then there exists $\epsilon>0, \eta>0$ and $\varphi \in \Phi_{\eta}$, such that for all $\overline{\mathbf{u}}$ in $\Omega$ and all $\mathbf{u}$ in the following intersection
\begin{equation*}
    \left\{\mathbf{u} \in \reals^{n}: \operatorname{dist}(\mathbf{u}, \Omega)<\epsilon\right\} \bigcap\left\{\mathbf{u} \in \reals^{n}: f(\overline{\mathbf{u}})<f(\mathbf{u})<f(\overline{\mathbf{u}})+\eta\right\},
\end{equation*}
the following inequality holds
\begin{equation*}
    \varphi^{\prime}(f(\mathbf{u})-f(\overline{\mathbf{u}})) \operatorname{dist}(0, \partial f(\mathbf{u}))>1.
\end{equation*}
\end{lemma}

Below we restate Theorem \ref{thm_joint_converge} for general $f, g$ and give a complete proof. Note that Assumption (\ref{asp_apg_l_smooth})-(\ref{asp_apg_coercive}) can ensure the assumptions below and thus Theorem \ref{thm_joint_converge} is a special case of Theorem \ref{thm_mapg_converge}.
\begin{theorem}
\label{thm_mapg_converge}
Consider the general problem
\begin{equation}
    \mathop{\min}\limits_{\bm{x}\in \reals^n} F(\bm{x})=f(\bm{x})+g(\bm{x})
\end{equation}
with the following assumptions
\begin{itemize}
    \item[A.1] $f(\bm{x})$ is proper and strongly smooth on any bounded sets.
    \item[A.2] $g(\bm{x})$ is proper, lower semicontinuous, bounded from below by $\underline{g}>-\infty$.
    \item[A.3] $F(\bm{x})$ is coercive, that is, $F(\bm{x})\rightarrow +\infty$ as $\Vert\bm{x}\Vert\rightarrow +\infty$.
\end{itemize}
Then there exist $\alpha_x,\alpha_y>0$ in \textbf{Algorithm} \ref{alg_APG} such that
\begin{itemize}
    \item[(\romannumeral1)] The sequences $\{\bm{x}_k\}$ and $\{\bm{v}_k\}$ are bounded, and $\mathop{\lim}\limits_{k\rightarrow\infty}\Vert \bm{v}_{k+1}-\bm{x}_k\Vert=0$.
    \item[(\romannumeral2)] The set of limit points of $\{\bm{x}_k\}$ is nonempty. If $\bm{x}^*$ is a limit point of $\{\bm{x}_k\}$, then $\bm{x}^*$ is a critical point of $F$.
    \item[(\romannumeral3)] If $F(\bm{x})$ satisfies KL property, and the desingularising function has the form of $\varphi(t)=\frac{C}{s}t^{s}$ for some $C>0, s\in (0,1]$, then
    \begin{itemize}
        \item[1.] If $s=1$, then there exists  $k_{1}$ such that  $F\left(\bm{x}_{k}\right)=F^{*}$ for all $k>k_{1}$ and the algorithm terminates in finite steps.
        
        \item[2.] If $s \in\left[\frac{1}{2}, 1\right)$, then there exists $k_{2}$ such that for all  $k>k_{2}$,
        \begin{equation*}
            F\left(\bm{x}_{k}\right)-F^{*} \leq\left(\frac{d_{1} C^{2}}{1+d_{1} C^{2}}\right)^{k-k_{2}} r_{k_{2}}.
        \end{equation*}
        
        \item[3.] If $s \in\left(0, \frac{1}{2}\right)$, then there exists $k_{3}$ such that for all $k>k_{3}$,
        \begin{equation*}
            F\left(\bm{x}_{k}\right)-F^{*} \leq\left(\frac{C}{\left(k-k_{3}\right) d_{2}(1-2 s)}\right)^{\frac{1}{1-2 s}}
        \end{equation*}
        where $F^{*}$ is the same function value at all the accumulation points of  $\left\{\bm{x}_{k}\right\}$, $r_k=F\left(\bm{v}_{k}\right)- F^{*}$, $d_{1}=\left(\frac{1}{\alpha_{x}}+L\right)^{2} /\left(\frac{1}{2 \alpha_{x}}-\frac{L}{2}\right)$, $d_{2}=\min \left\{\frac{1}{2 d_{1} C}, \frac{C}{1-2 s}\left(2^{\frac{2 s-1}{2 s-2}}-1\right) r_{0}^{2 s-1}\right\}$.
    \end{itemize}

\end{itemize}

\end{theorem}

\begin{proof}
Let the sublevel set $B_0:=\{\bm{x}:F(\bm{x})\leq F(\bm{x}_0)\}$, which is a bounded set by coerciveness assumption. Then there exits $G_f,G_g<+\infty$ such that $\Vert\nabla f(\bm{x})\Vert \leq G_f$ for any $\bm{x}\in B_0$ and $g(\bm{x})-\underline{g}\leq G_g$, for any $\bm{x}\in B_0\cap \textbf{dom}\ g$. Let $\alpha_0>0$ be any positive constant, and $B_1:=\textbf{conv}\{\bm{x}: \operatorname{dist}(\bm{x},B_0)\leq \sqrt{2\alpha_0G_g}+2\alpha_0G_f\}$. Then $B_1$ is also a bounded set and thus $f$ has strong smooth coefficient $L$ on $B_1$. We show that $\alpha_x=\alpha_y<\min\{\alpha_0, \frac{1}{L}\}$ suffices. 

First we show by induction that, $F(\bm{x}_{k})-F(\bm{x}_{k-1})\leq -\delta\Vert \bm{v}_{k}-\bm{x}_{k-1}\Vert^2$ for $\delta=\frac{1}{2\alpha_x}-\frac{L}{2}>0$ and thus $\bm{x}_k$ belongs to $B_0$ for all $k$.

The case for $k=0$ is trivial.

Suppose that $\bm{x}_k$ belongs to $B_0$. Then by construction of $\bm{v}_{k+1}$, we have 
\begin{equation}
\label{pf_sub_grad}
    0\in \partial g(\bm{v}_{k+1})+\frac{\bm{v}_{k+1}-(\bm{x}_k-\alpha_x\nabla f(\bm{x}_k))}{\alpha_x},
\end{equation}
where $\partial g$ denotes the limiting $g$-subdifferential.

This implies 
\begin{equation*}
    g(\bm{v}_{k+1})+\frac{1}{2\alpha_x}\Vert \bm{v}_{k+1}-(\bm{x}_k-\alpha_x\nabla f(\bm{x}_k))\Vert^2\leq g(\bm{x}_k)+\frac{\alpha_x}{2}\Vert \nabla f(\bm{x}_k)\Vert^2,
\end{equation*}
\begin{equation}
\begin{aligned}
     \Vert \bm{v}_{k+1}-\bm{x}_k\Vert
    &\leq \sqrt{2\alpha_x(g(\bm{x}_k)-\underline{g})+\alpha_x^2\Vert \nabla f(\bm{x}_k)\Vert^2}+\alpha_x\Vert \nabla f(\bm{x}_k)\Vert\\
    &\leq \sqrt{2\alpha_0G_g}+2\alpha_0G_f.
\end{aligned}
\end{equation}
Therefore $\bm{v}_{k+1}$ belongs to $B_1$. For Lipschitz differentiable $f$ on the convex set $B_1$ we have
\begin{equation}
\label{pf_suff_des}
\begin{aligned}
    F(\bm{v}_{k+1})
    &\leq g(\bm{v}_{k+1})+f(\bm{x}_k)+\langle \nabla f(\bm{x}_k),\bm{v}_{k+1}-\bm{x}_k\rangle+\frac{L}{2}\Vert \bm{v}_{k+1}-\bm{x}_k\Vert^2\\
    &\leq F(\bm{x}_k)-(\frac{1}{2\alpha_x}-\frac{L}{2})\Vert \bm{v}_{k+1}-\bm{x}_k\Vert^2.
\end{aligned}
\end{equation}
Hence 
\begin{equation}
\label{pf_suf_dscnt}
    F(\bm{x}_{k+1})\leq F(\bm{v}_{k+1})\leq F(\bm{x}_k)-\delta\Vert \bm{v}_{k+1}-\bm{x}_k\Vert^2.
\end{equation}
By induction assumption, we have $\bm{x}_{k+1}, \bm{v}_{k+1}$ belongs to $B_0$.

The following procedures are the same as in \cite{NIPS2015_f7664060}.

\begin{itemize}
    \item[(\romannumeral1)] The claims above verifies the boundedness of $\{\bm{x}_k\}$ and $\{\bm{v}_k\}$. Adding up \ref{pf_suf_dscnt} from $1$ to $k$, we have
    \begin{equation}
        \delta\sum_{i=1}^{k}\Vert \bm{v}_{i+1}-\bm{x}_i\Vert^2\leq F(\bm{x}_{1})-F(\bm{x}_{k+1})\leq F(\bm{x}_{1})-F^*,
    \end{equation}
    and thus $\sum_{i=1}^{\infty}\Vert \bm{v}_{i+1}-\bm{x}_i\Vert^2<+\infty$, in particular, $\mathop{\lim}\limits_{k\rightarrow\infty}\Vert \bm{v}_{k+1}-\bm{x}_k\Vert=0$.
    
    \item[(\romannumeral2)] Since $\{\bm{x}_k\}$ is bounded, its limit point set must be nonempty. Assume there is a subsequence $\{\bm{x}_{k_j}\}$ that converges to $\bm{x}^*$. Then from (\romannumeral1) we have $\Vert \bm{v}_{k_j+1}-\bm{x}^*\Vert\rightarrow 0$. Note that by construction of $\bm{v}_{k}$,
    \begin{equation*}
        g(\bm{v}_{k_j+1})+\frac{1}{2\alpha_x}\Vert \bm{v}_{k_j+1}-(\bm{x}_{k_j}-\alpha_x\nabla f(\bm{x}_{k_j}))\Vert^2\leq
        g(\bm{x}^*)+\frac{1}{2\alpha_x}\Vert \bm{x}^*-(\bm{x}_{k_j}-\alpha_x\nabla f(\bm{x}_{k_j}))\Vert^2.
    \end{equation*}
    So 
    \begin{equation*}
        \mathop{\limsup}\limits_{j\rightarrow\infty} g(\bm{v}_{k_j+1})\leq g(\bm{x}^*).
    \end{equation*}
    Because $g$ is lower semicontinuous, the inequality above is actually an equality, i.e.
    \begin{equation}
        \mathop{\lim}\limits_{j\rightarrow\infty}g(\bm{v}_{k_j+1})= g(\bm{x}^*).
    \end{equation}
    Hence $F(\bm{v}_{k_j+1})\rightarrow F(\bm{x}^*)$ and by (\ref{pf_sub_grad}), we have
    \begin{equation}
        0\in \partial F(\bm{x}^*).
    \end{equation}
    
    \item[(\romannumeral3)] From (\ref{pf_sub_grad}), we have
    \begin{equation}
        \label{pf_kl_ub}
        \operatorname{dist}(0,\partial F(\bm{v}_{k+1}))\leq \Vert \nabla f(\bm{x}_k)-\nabla f(\bm{v}_{k+1})+\frac{1}{\alpha_x}(\bm{v}_{k+1}-\bm{x}_k) \Vert \leq (\frac{1}{\alpha_x}+L)\Vert \bm{v}_{k+1}-\bm{x}_k\Vert.
    \end{equation}
    If there exists $\bar{k}$ such that  $F\left(\bm{v}^{\bar{k}}\right)=F^{*}$, then  $F\left(\bm{v}^{\bar{k}}\right)=F\left(\bm{v}^{\bar{k}+1}\right)=\cdots=F^{*}$. So $\| \bm{v}^{\bar{k}+1}-   \bm{x}^{\bar{k}}\|=\| \bm{v}^{\bar{k}+2}-\bm{x}^{\bar{k}+1} \|=\cdots=0$. The conclusion holds. If  $F\left(\bm{v}_{k}\right)>F^{*}$  for all $k$, then from  $F\left(\bm{v}_{k}\right) \rightarrow F^{*}$ we know that there exists  $\hat{k}_{1}$  such that  $F\left(\bm{v}_{k}\right)<F^{*}+\eta$ whenever  $k>\hat{k}_{1}$. On the other hand, because  $\operatorname{dist}\left(\bm{v}_{k}, \Omega\right) \rightarrow 0$, there exists $\hat{k}_{2}$ such that  $\operatorname{dist}\left(\bm{v}_{k}, \Omega\right)<\varepsilon$ whenever $k>\hat{k}_{2}$. Let  $k>k_{0}=\max \left\{\hat{k}_{1}, \hat{k}_{2}\right\}$, we have
    \begin{equation}
        \bm{v}_{k} \in\{\bm{v}: \operatorname{dist}(\bm{v}, \Omega) \leq \varepsilon\} \bigcap\{F^{*}<F(\bm{v})<F^{*}+\eta\}.
    \end{equation}
    From the uniform KL property in Lemma \ref{lemma_kl}, there exists a concave function  $\varphi$ such that
    \begin{equation}
        \label{pf_kl_ineq}
        \varphi^{\prime}\left(F\left(\bm{v}_{k}\right)-F^{*}\right) \operatorname{dist}\left(0, \partial F\left(\bm{v}_{k}\right)\right) \geq 1.
    \end{equation}
    
    Define $r_{k}=F\left(\bm{v}_{k}\right)-F^{*}$. WLOG, We suppose that $r_{k}>0$ for all  sufficiently large $k$. Otherwise $F\left(\bm{v}_{k}\right)=F\left(\bm{v}_{k+1}\right)=\cdots=F^{*}$ and the algorithm terminates in finite steps. From (\ref{pf_kl_ub}), (\ref{pf_kl_ineq}) and (\ref{pf_suf_dscnt}) we have
    \begin{equation*}
        \begin{aligned}
            1 & \leq\left[\varphi^{\prime}\left(F\left(\bm{v}_{k}\right)-F^{*}\right) \operatorname{dist}\left(0, \partial F\left(\bm{v}_{k}\right)\right)\right]^{2} \\
            & \leq\left[\varphi^{\prime}\left(r_{k}\right)\right]^{2}\left(\frac{1}{\alpha_{x}}+L\right)^{2}\left\|\bm{v}_{k}-\mathbf{x}_{k-1}\right\|^{2} \\
            & \leq\left[\varphi^{\prime}\left(r_{k}\right)\right]^{2}\left(\frac{1}{\alpha_{x}}+L\right)^{2} \frac{F\left(\bm{v}_{k-1}\right)-F\left(\bm{v}_{k}\right)}{\left(\frac{1}{2 \alpha_{x}}-\frac{L}{2}\right)} \\
            &=d_{1}\left[\varphi^{\prime}\left(r_{k}\right)\right]^{2}\left(r_{k-1}-r_{k}\right),
        \end{aligned}
    \end{equation*}
    for all  $k>k_{0}$, where $d_{1}=\left(\frac{1}{\alpha_{x}}+L\right)^{2} /\left(\frac{1}{2 \alpha_{x}}-\frac{L}{2}\right)$. Because $\varphi$ has the form of  $\varphi(t)=\frac{C}{s} t^{s}$, we have $\varphi^{\prime}(t)=C t^{s-1}$. So it becomes
    \begin{equation}
        \label{pf_kl_diff}
        1 \leq d_{1} C^{2} r_{k}^{2 s-2}\left(r_{k-1}-r_{k}\right).
    \end{equation}
 
    \paragraph{1. Case $s=1$.}
    
    In this case, (\ref{pf_kl_diff}) becomes
    \begin{equation}
        1 \leq d_{1} C^{2}\left(r_{k}-r_{k+1}\right) .
    \end{equation}
    
    Because $r_{k} \rightarrow 0$ and $d_{1}>0, C>0$, this is a contradiction. So there exists $k_{1}$ such that $r_{k}=0$ for all $k>k_{1}$. The algorithm terminates in finite steps.
    
    \paragraph{2. Case $s\in\left[\frac{1}{2}, 1\right)$.}

    In this case,  $0<2-2 s \leq 1$. As $r_{k} \rightarrow 0$, there exists $\hat{k}_{3}$ such that $r_{k}^{2-2 s} \geq r_{k}$ for all $k>\hat{k}_{3}$. (\ref{pf_kl_diff}) becomes
    \begin{equation*}
        r_{k} \leq d_{1} C^{2}\left(r_{k-1}-r_{k}\right).
    \end{equation*}
    So we have for all $k_{2}>\max \left\{k_{0}, \hat{k}_{3}\right\}$
    \begin{equation*}
        r_{k} \leq\left(\frac{d_{1} C^{2}}{1+d_{1} C^{2}}\right)^{k-k_{2}} r_{k_{2}}.
    \end{equation*}
    and
    \begin{equation}
        F\left(\mathrm{x}_{k}\right)-F^{*} \leq F\left(\mathbf{v}_{k}\right)-F^{*}=r_{k} \leq\left(\frac{d_{1} C^{2}}{1+d_{1} C^{2}}\right)^{k-k_{2}} r_{k_{2}}.
    \end{equation}
    
    \paragraph{3. Case $s \in\left(0, \frac{1}{2}\right)$.}
    
    In this case, $2 s-2 \in(-2,-1), 2 s-1 \in(-1,0)$. As  $r_{k-1}>r_{k}$, we have $r_{k-1}^{2 s-2}<r_{k}^{2 s-2}$ and  $r_{0}^{2 s-1}<\cdots<r_{k-1}^{2 s-1}<r_{k}^{2 s-1}$.
    Define $\phi(t)=\frac{C}{1-2 s} t^{2 s-1}$, then  $\phi^{\prime}(t)=-C t^{2 s-2}$.
    
    If $r_{k}^{2 s-2} \leq 2 r_{k-1}^{2 s-2}$, then
    \begin{equation*}
        \begin{aligned}
        \phi\left(r_{k}\right)-\phi\left(r_{k-1}\right)
        &=C \int_{r_{k}}^{r_{k-1}} t^{2 s-2} d t \\
        &\geq C\left(r_{k-1}-r_{k}\right) r_{k-1}^{2 s-2} \\
        &\geq \frac{C}{2}\left(r_{k-1}-r_{k}\right) r_{k}^{2 s-2}, \\
        \end{aligned}
    \end{equation*}
    for all $k>k_{0}$.
    
    If $r_{k}^{2 s-2} \geq 2 r_{k-1}^{2 s-2}$, then $r_{k}^{2 s-1} \geq 2^{\frac{2 s-1}{2 s-2}} r_{k-1}^{2 s-1}$, and
    \begin{equation*}
        \begin{aligned}
        \phi\left(r_{k}\right)-\phi\left(r_{k-1}\right) &=\frac{C}{1-2 s}\left(r_{k}^{2 s-1}-r_{k-1}^{2 s-1}\right) \\
        & \geq \frac{C}{1-2 s}\left(2^{\frac{2 s-1}{2 s-2}}-1\right) r_{k-1}^{2 s-1} \\
        &=q r_{k-1}^{2 s-1} \\
        &\geq q r_{0}^{2 s-1}, \\
        \end{aligned}
    \end{equation*}
    where $q=\frac{C}{1-2 s}\left(2^{\frac{2 s-1}{2 s-2}}-1\right)$.
    
    Let $d_{2}=\min \left\{\frac{1}{2 d_{1} C}, q r_{0}^{2 s-1}\right\}$, we have
    \begin{equation*}
        \phi\left(r_{k}\right)-\phi\left(r_{k-1}\right) \geq d_{2},
    \end{equation*}
    for all  $k>k_{0}$ and
    \begin{equation*}
        \phi\left(r_{k}\right) \geq \phi\left(r_{k}\right)-\phi\left(r_{k_{0}}\right) \geq \sum_{i=k_{0}+1}^{k} \phi\left(r_{i}\right)-\phi\left(r_{i-1}\right) \geq\left(k-k_{0}\right) d_{2}.
    \end{equation*}
    Hence,
    \begin{equation*}
         r_{k} \leq\left(\frac{C}{\left(k-k_{0}\right) d_{2}(1-2 s)}\right)^{\frac{1}{1-2 s}},
    \end{equation*}
    and
    \begin{equation}
         F\left(\mathbf{x}_{k}\right)-F^{*} \leq F\left(\mathbf{v}_{k}\right)-F^{*}=r_{k} \leq\left(\frac{C}{\left(k-k_{3}\right) d_{2}(1-2 s)}\right)^{\frac{1}{1-2 s}},
    \end{equation}
    which completes the proof.

\end{itemize}

\end{proof}

\section{Technical details}
\setcounter{equation}{0}
\setcounter{theorem}{0}
\label{appendix_notes}

\subsection{Derivation of gradient}
\label{subsec:grad}

Consider $f(\theta, W)=\sum_k l_k(\theta_k)+\frac{1}{2}\sum_{i<j}W_{ij}\Vert \theta_i - \theta_j\Vert_{2}^{2}+\frac{\lambda_2(1-\eta)}{2}\Vert W-W_0 \Vert_{F}^{2}-\lambda_1 \log{\det{(\mu I+\mathcal{G}(W))}}$. We want to compute the gradient with respect to $W$. Although $W$ must be symmetric and diagonal-free in the algorithm, we think it has $K(K-1)$ variables when deriving its gradient.

Rewrite the first term $f_1:=\frac{1}{2}\sum_{i<j}W_{ij}\Vert \theta_i - \theta_j\Vert_{2}^{2}=\frac{1}{4}\sum_{i\neq j}W_{ij}\Vert \theta_i - \theta_j\Vert_{2}^{2}$. It has gradient $\nabla f_1=(\Vert \theta_i - \theta_j\Vert_{2}^{2} / 4)_{ij}$. 

The second term $f_2:=\frac{\lambda_2(1-\eta)}{2}\Vert W-W_0 \Vert_{F}^{2}$ has gradient $\nabla f_2=\lambda_2(1-\eta)(W-W_0)$. Note that this is valid because the diagonal elements are all zero. 

After tedious calculus, we can derive the last term $f_3:=-\lambda_1 \log{\det{(\mu I+\mathcal{G}(W))}}$ has gradient $\nabla f_3=-\lambda_1 P\left((\mu I+\mathcal{G}(W))^{-1}\right)$, where $[P(L)]_{ij}=L_{ii}-L_{ij}$. Again, the diagonal elements are all zero.

Finally we obtain $\nabla_{W} f(\theta, W)=\nabla f_1+\nabla f_2+\nabla f_3$, which is not symmetric generally but is diagonal-free.

\subsection{Derivation of proximal operator}
Let $g(W)=\lambda_2\eta\Vert W\Vert_{1}+\sum_{i,j}\mathcal{I}_{W_{ij}=W_{ji}\geq 0}+\mathcal{I}_{\textbf{diag} (W)=0}$. We want to find $\prox_{\alpha g}(U) := \mathop{\arg\min}\limits_{W} g(W)+\frac{1}{2\alpha}\Vert W-U\Vert_{F}^{2}$. Since $W_{ij}=W_{ji}$, we take them as a single variable. The KKT condition is that 
\begin{equation*}
    0\in \lambda_2 \eta \alpha \partial \Vert W_{ij}\Vert_1 + \partial \mathcal{I}_{W_{ij}\geq 0} + (W_{ij} - \frac{U_{ij}+U_{ji}}{2}),\ \text{ for } i<j.
\end{equation*}
Therefore $W_{ij}=W_{ji}= \max{(0,\frac{U_{ij}+U_{ji}}{2}-\lambda_2 \alpha \eta)}$ for $i<j$ and $W_{ii}=0$. A key observation is that in our MAPG algorithm, $U$ has the form of $\widetilde{W}-\alpha\nabla f(\widetilde{W})$. And $U$ is diagonal-free as long as $\widetilde{W}$ is (see \ref{subsec:grad}). So the expression $\prox_{\alpha g}(K)=\max{(0, \frac{U+U^T}{2}-\lambda_2 \alpha \eta)}$ is valid, as long as $W_0$ and the initial input $W^{(0)}$ are diagonal-free.

\end{document}